\documentclass[twoside]{article}

\usepackage{amsmath}
\usepackage{amsthm}
\usepackage{amssymb}
\usepackage{mathtools}
\usepackage[cal=euler,scr=boondoxupr,bb=bboldxLight]{mathalpha}
\usepackage{upgreek}
\usepackage{mleftright}
\usepackage[sf,bf]{titlesec}
\usepackage{bm}
\usepackage{euscript}
\usepackage{url}            % simple URL typesetting
\usepackage{booktabs}       % professional-quality tables
\usepackage{amsfonts}       % blackboard math symbols
\usepackage{nicefrac}       % compact symbols for 1/2, etc.
\usepackage{microtype}      % microtypography
\usepackage{color}
\usepackage{tabu}
\usepackage{multirow}
\usepackage{float}
\usepackage{booktabs}
\usepackage{algorithm}
\usepackage{algorithmicx}
\usepackage[noEnd=false,indLines=true]{algpseudocodex}
\usepackage{color}
\usepackage{tabu}
\usepackage{comment}
\usepackage{graphicx}
\usepackage{epstopdf}
\usepackage{subcaption}
\usepackage{framed}
\usepackage{enumitem}
\usepackage{textcomp}
\usepackage{setspace}
\usepackage[pdftex]{hyperref}
\usepackage[capitalize,noabbrev]{cleveref}
\usepackage{latexsym}
\usepackage{tcolorbox}
\usepackage[round]{natbib}
\usepackage{alltt}
\usepackage{stmaryrd}
\usepackage{xspace}
\usepackage{euscript}
\usepackage{xfrac}    % for \xfrac macro
\usepackage{nicefrac} % for \nicefrac macro
\usepackage{tikz}
\usepackage{tikz-qtree}
\usepackage{todonotes}
\usepackage{diagbox}
\usepackage{wrapfig}
\usepackage[normalem]{ulem}
\usepackage{inconsolata} % for tt font

\usepackage[margin=1in]{geometry}
\usepackage{fancyhdr}
\fancypagestyle{plain}{\fancyhf{} 
	
	\fancyfoot[C]{\textsf{\thepage}}
}
\definecolor{mydarkblue}{rgb}{0,0.08,0.45}
\hypersetup{ %
	colorlinks=true,
	linkcolor=mydarkblue,
	citecolor=mydarkblue,
	filecolor=mydarkblue,
	urlcolor=mydarkblue}

\newcommand{\calB}{\mathcal{B}}

\newcommand{\calD}{\mathcal{D}}

\newcommand{\calF}{\mathcal{F}}

\newcommand{\calZ}{\mathcal{Z}}

\newcommand{\scrO}{\mathscr{O}}

\newcommand{\Var}{\mathrm{Var}}

\newcommand{\Ex}{\mathbb{E}}
\newcommand{\Prob}{\mathbb{P}}

\newcommand{\RR}{\mathbb{R}}

\newcommand{\NN}{\mathbb{N}}

\newcommand{\e}{\mathrm{e}}

\DeclareMathOperator*{\argmin}{argmin}

\DeclareMathOperator*{\minimize}{minimize}

\newcommand{\sumK}{\sum_{k=1}^K}

\newcommand{\sumM}{\sum_{m=1}^M}

\newcommand{\sfT}{\mathsf{T}}

\newcommand{\xbar}{\overline{x}}

\renewcommand{\mid}{\,|\,}
\newcommand{\midd}{\,|\kern-0.25ex|\,}
\newcommand{\setn}{\llbracket n\rrbracket}

\newcommand{\setM}{\llbracket M\rrbracket}
\newcommand{\set}[1]{\llbracket #1\rrbracket}

\newcommand{\dotp}[2]{\langle #1, #2\rangle}

\newcommand{\Probhat}{\widehat{\mathbb{P}}}

\DeclareMathAlphabet\rsfscr{U}{rsfso}{m}{n}

\let\le\leqslant
\let\ge\geqslant

\let\hat\widehat
\let\tilde\widetilde

\makeatletter
\DeclareFontFamily{OMX}{MnSymbolE}{}
\DeclareSymbolFont{MnLargeSymbols}{OMX}{MnSymbolE}{m}{n}
\SetSymbolFont{MnLargeSymbols}{bold}{OMX}{MnSymbolE}{b}{n}
\DeclareFontShape{OMX}{MnSymbolE}{m}{n}{
	<-6>  MnSymbolE5
	<6-7>  MnSymbolE6
	<7-8>  MnSymbolE7
	<8-9>  MnSymbolE8
	<9-10> MnSymbolE9
	<10-12> MnSymbolE10
	<12->   MnSymbolE12
}{}
\DeclareFontShape{OMX}{MnSymbolE}{b}{n}{
	<-6>  MnSymbolE-Bold5
	<6-7>  MnSymbolE-Bold6
	<7-8>  MnSymbolE-Bold7
	<8-9>  MnSymbolE-Bold8
	<9-10> MnSymbolE-Bold9
	<10-12> MnSymbolE-Bold10
	<12->   MnSymbolE-Bold12
}{}

\let\llangle\@undefined
\let\rrangle\@undefined
\DeclareMathDelimiter{\llangle}{\mathopen}%
{MnLargeSymbols}{'164}{MnLargeSymbols}{'164}
\DeclareMathDelimiter{\rrangle}{\mathclose}%
{MnLargeSymbols}{'171}{MnLargeSymbols}{'171}
\makeatother

\renewcommand{\left}{\mleft}
\renewcommand{\right}{\mright}

\newcommand{\iid}{i.i.d.\xspace}

\newcommand{\Adam}{\textsc{Adam}\xspace}
\newcommand{\AdamW}{\textsc{AdamW}\xspace}
\newcommand{\SGD}{\textsc{SGD}\xspace}

\newcommand{\SHB}{\textsc{SHB}\xspace}

\newcommand{\AdaGrad}{\textsc{AdaGrad}\xspace}

\newcommand{\ResNet}{\textsc{ResNet}\xspace}

\theoremstyle{definition}
\newtheorem{assumption}{Assumption}
\newtheorem{definition}{Definition}

\theoremstyle{plain}
\newtheorem{theorem}{Theorem}%[section]
\newtheorem{lemma}{Lemma}%[theorem]
\theoremstyle{remark}
\newtheorem{remark}{Remark}%[section]

\crefname{assumption}{Assumption}{Assumptions}
\Crefname{assumption}{Assumption}{Assumptions}
\crefname{problem}{Problem}{Problems}
\Crefname{problem}{Problem}{Problems}
\crefname{example}{Example}{Examples}
\Crefname{example}{Example}{Examples}

\newcommand{\norm}[1]{\left\lVert#1\right\rVert}

\newcommand{\matsnorm}[2]{\lvert\kern-0.25ex\lvert\kern-0.25ex\lvert #1 \rvert\kern-0.25ex\rvert\kern-0.25ex\rvert_{#2}}

\newcommand{\out}{\mathsf{out}}

\newcommand{\gkm}{g_k^m}
\newcommand{\xkm}{x_k^m}
\newcommand{\xbark}{\xbar_k}
\newcommand{\xbarkpo}{\xbar_{k+1}}
\newcommand{\xstar}{x^\star}
\newcommand{\gradFxkm}{\nabla F(x_k^m)}
\newcommand{\gradFxbark}{\nabla F(\xbar_k)}
\newcommand{\Fstar}{F^\star}
\newcommand{\gbar}{\overline{g}}
\newcommand{\gbark}{\gbar_k}

\begin{document}
    \title{\sffamily Communication-Efficient Adaptive Batch Size Strategies \\for Distributed Local Gradient Methods
    }
    \author{
        Tim Tsz-Kit Lau%        
        \thanks{Department of Biostatistics, Epidemiology and Informatics, Perelman School of Medicine, University of Pennsylvania, Philadelphia, PA 19104, USA; Email: 
        \href{mailto:timlautk@upenn.edu}{\texttt{timlautk@upenn.edu}}. Part of the work of Tim Tsz-Kit Lau was performed at The University of Chicago Booth School of Business. }
        \and
        Weijian Li%
        \thanks{Department of Computer Science, Northwestern University, Evanston, IL 60208, USA; Email: \href{mailto:weijianli2021@u.northwestern.edu}{\texttt{weijianli2021@u.northwestern.edu}}, \href{mailto:hanliu@northwestern.edu}{\texttt{hanliu@northwestern.edu}}.}
        \and
        Chenwei Xu%
        \thanks{Department of Statistics and Data Science, Northwestern University, Evanston, IL 60208, USA; \href{mailto:chenweixu2023@u.northwestern.edu}{\texttt{chenweixu2023@u.northwestern.edu}}.}
        \and 
        Han Liu%
        \footnotemark[2]
        \footnotemark[3]
        \and Mladen Kolar%        
        \thanks{Department of Data Sciences and Operations, University of Southern California Marshall School of Business, Los Angeles, CA 90089, USA; Email: \href{mailto:mkolar@marshall.usc.edu}{\texttt{mkolar@marshall.usc.edu}}. Part of the work of Mladen Kolar was performed at The University of Chicago Booth School of Business. }
    }
    
    \maketitle

\begin{abstract} 
    Modern deep neural networks often require distributed training with many workers due to their large size. As the number of workers increases, communication overheads become the main bottleneck in data-parallel minibatch stochastic gradient methods with per-iteration gradient synchronization. Local gradient methods like Local \SGD reduce communication by only synchronizing model parameters and/or gradients after several local steps. Despite an understanding of their convergence and the importance of batch sizes for training efficiency and generalization, optimal batch sizes for local gradient methods are difficult to determine. We introduce adaptive batch size strategies for local gradient methods that increase batch sizes adaptively to reduce minibatch gradient variance. We provide convergence guarantees under homogeneous data conditions and support our claims with image classification and language modeling experiments, demonstrating the effectiveness of our strategies for both training efficiency and generalization. 
\end{abstract}

\section{Introduction}
\label{sec:intro}
Recent advances in deep learning, particularly large language models (LLMs) with more than billions of parameters, necessitate efficient distributed training on many workers such as GPUs or TPUs. Data-parallel optimization \citep{zinkevich2010parallelized} simplifies this by having each worker compute and average batch gradients from different \iid~data batches, followed by a global model update. Despite its efficiency with numerous workers, this method faces challenges such as costly gradient synchronization and reduced generalization from large batches. Frequent synchronization of model parameters or gradients among workers leads to significant communication overheads, exacerbated as worker counts and model sizes increase \citep{tang2021one,li2022one,xu2023slamb}. This challenge intensifies with workers distributed across nodes with limited inter-node communication speeds. To mitigate this, communication-efficient distributed optimization methods \citep{lan2020communication}, such as local gradient methods \citep{stich2019local, zhou2018convergence_ijcai, yu2019parallel, haddadpour2019local, wang2019adaptive} and their variants, have been developed. In local stochastic gradient methods (Local \SGD), workers train locally and synchronize less frequently by averaging model parameters after several local steps, effectively reducing communication overheads. In data-parallel optimization, the total batch size is the sum of all workers' local batch sizes, increasing with more workers. Attributed to the ever-rising memory of workers, large batch sizes are preferred to fully utilize worker memory, making large-batch training the \emph{de facto} paradigm in large-scale model training. However, this approach often results in worse model generalization, a problem noted by \citet{lecun2002efficient}. Large-batch training typically creates a loss landscape with sharp minima, worsening generalization \citep{keskar2017on}. Although scaling the learning rate with the batch size can mitigate the so-called \emph{generalization gap} \citep{hoffer2017train,smith2018bayesian,smith2018dont}, it cannot fully close it. \citet{keskar2017on} recommended adaptive sampling to gradually increase batch sizes during training \citep{byrd2012sample,friedlander2012hybrid}, explored in \citet{de2016big,de2017automated,lau2024adadagrad,ostroukhov2024adabatchgrad}. However, this method is still confined to single- or few-worker scenarios and does not scale well with larger models and more workers. 

Motivated by the above two challenges, we answer the following question affirmatively: 
% \vspace*{-5mm}
\begin{center}
    \emph{Can we develop communication-efficient and memory-efficient training strategies with desirable generalization performance for large-scale distributed training? }
\end{center}    
    
\paragraph{Contributions.}
We make the following contributions: (i) We develop adaptive batch size strategies for \emph{local gradient methods} by extending the adaptive sampling methods in \citet{byrd2012sample,bollapragada2018adaptive}. Such strategies are for local batch sizes, i.e., batch sizes of individual workers. This is particularly relevant when the workers are heterogeneous devices with different computational speeds and memories, as well as different local objectives and datasets. The proposed strategies are \emph{communication-efficient} by leveraging local gradient methods, and \emph{memory-efficient} by leveraging large batch sizes at later stages of training. (ii) Under various standard assumptions, we provide convergence guarantees of the proposed approaches for smooth (strongly) convex and nonconvex objectives under a homogeneous data setting. (iii) We empirically demonstrate the efficacy of the proposed strategies with \ResNet for image classification on the CIFAR-10 and ImageNet datasets and MicroLlama on the C4 dataset respectively, in terms of communication efficiency, memory efficiency, and model generalization performance. To grasp a thorough understanding of the strategies, we also study the effect of different hyperparameters such as the number of local steps and the one which controls the probability of increasing batch sizes. 

\section{Related Work}    

\paragraph{Minibatch gradient methods.}
\emph{Minibatch SGD} \citep{dean2012large, robbins1951, bottou2010large, dekel2012optimal} is the simplest and most widely used optimizer for large-scale machine learning \citep{bottou2018optimization}, and is easily extendable with data parallelism~\citep{li2020pytorch}. The convergence of minibatch \SGD and its variants like \AdaGrad \citep{duchi2011adagrad} and \Adam \citep{kingma2015} is well-understood across various settings, including convex \citep{gower2019sgd}, nonconvex \citep{khaled2023better}, and nonsmooth \citep{davis2020stochastic} objectives, under different gradient variance assumptions. However, frequent gradient synchronization among workers leads to communication overheads, posing a significant challenge in distributed training.

%\vspace*{-2.5mm}
\paragraph{Local gradient methods.} 
Local SGD \citep{stich2019local, zinkevich2010parallelized, mcmahan2017communication}, a communication-efficient variant of minibatch SGD, reduces overheads from frequent synchronization in distributed training. Recent studies \citep{liu2024asynchronous,douillard2023diloco} highlight its efficacy in pre-training large language models with extensive computational resources. Optimization of Local SGD for varied data settings has shown promising convergence rates \citep{stich2019local, woodworth2020local, woodworth2020minibatch, koloskova2020unified, khaled2020tighter, stich2020error, wang2021cooperative}. It includes extensions like asynchronous SGD \citep{mishchenko2022asynchronous, koloskova2022sharper, nguyen2022federated, leconte2024queuing, islamov2024asgrad} and hierarchical SGD \citep{lin2020dont, castiglia2021multilevel, wang2022demystifying}. Federated learning \citep{mcmahan2017communication, chen2021communication, kairouz2021advances, karimireddy2020scaffold, reddi2021adaptive}, another model focusing on efficiency, addresses heterogeneity and partial participation. While applicable, we defer exploring federated learning's theoretical and empirical aspects to future work.

%\vspace*{-2.5mm}
\paragraph{Generalization gap in large-batch training.} Large-batch training not only faces communication overheads in minibatch \SGD under distributed data parallelism but also shows poorer generalization despite better hardware efficiency. It has been successful in tasks like ImageNet classification and language model training \citep{you2020large,goyal2017accurate,akiba2017extremely,shallue2019measuring,liu2019roberta,puri2018large}. Addressing these, \citet{lin2020dont} introduced \emph{Post-local SGD}, a hybrid of minibatch and Local \SGD, where early training involves frequent synchronization of local models, followed by Local \SGD. In large-scale environments, \citet{ortiz2021trade} noted that Post-local \SGD does not consistently outperform Local \SGD in generalization, and the transition from minibatch to Local \SGD presents a trade-off between communication efficiency and generalization. Following these findings, \citet{gu2024quadratic,gu2023why} analyzed Local \SGD's generalization via an SDE approximation and proposed the Quadratic Synchronization Rule (QSR) for adjusting local gradient steps during training.

%\vspace*{-2.5mm}
\paragraph{Adaptive batch size strategies.}
In stochastic optimization, adaptive sampling methods \citep{byrd2012sample,friedlander2012hybrid,bollapragada2018adaptive} adjust batch sizes based on gradient noise, enhancing training processes for smooth finite-sum problems. These methods, further explored in deep learning \citep{de2016big,de2017automated,lau2024adadagrad,ostroukhov2024adabatchgrad}, have not been applied to data parallelism with distributed batches. The strategies discussed here draw from batch size scheduling, a key technique in pretraining LLMs, evidenced by its use in GPT-3 \citep{brown2020language}, Nemotron-4 \citep{parmar2024nemotron}, OLMo 65B \citep{groeneveld2024olmo}, and DeepSeek-V2 \citep{deepseekai2024deepseekv2}, which involve increasing batch sizes in set stages. Despite their benefits for efficiency and parallelization, these strategies are heuristic and have unclear impacts on training. This work aims to develop theoretically grounded, communication-efficient adaptive batch size strategies akin to adaptive sampling, viewed as stochastic variance-reduction methods \citep{johnson2013accelerating,xiao2014proximal,gower2020variance}, focusing on scalable batch size increments.

Several studies have explored adaptive batch size strategies in deep learning, including Big Batch \SGD \citep{de2016big,de2017automated}, CABS \citep{balles2017coupling}, AdaBatch \citep{devarakonda2017adabatch}, and SimiGrad \citep{qin2021simigrad}. Conversely, AdaScale \SGD \citep{johnson2020adascale} modifies learning rates for large batches instead of batch sizes. These methods, however, are often not theoretically robust, lacking solid convergence proofs or are restricted to \SGD analyses under strict conditions like convexity or the Polyak--\L{}ojasiewicz condition. Their batch size adjustment strategies usually depend on heuristic rules, such as geometric growth or decay \citep{qin2021simigrad}, which can lead to unreliable outcomes due to the lack of theoretical guarantees. Crucially, these methods overlook the distributed data-parallel context, prevalent in practice, thus failing to address the needs of adaptive batch size strategies in large-scale distributed training.

%\vspace*{-2.5mm}
\paragraph{Hyperparameter tuning for Local \SGD.}
Besides batch sizes, recent research has also focused on optimizing learning rate schedules for Local \SGD, particularly in identical data settings as discussed in \citet{balles2024on}. This optimization is crucial in federated learning with heterogeneous data, where local objectives differ significantly in geometry and smoothness, necessitating locally adaptive learning rate schedules \citep{mukherjee2024locally,kim2024adaptive}. Our adaptive local batch size strategies complement these efforts.

\section{Preliminaries}
\paragraph{Notation.} 
We define $\setn \coloneqq \{1, \ldots, n\}$ for $n\in\NN^*\coloneqq\NN\setminus\{0\}$, the set of positive integers. We also write $\set{m,n} \coloneqq \{m,m+1, \ldots, n\}$ for the set of integers ranging from $m$ to $n$ (inclusive) for $m<n$. We denote the Euclidean inner product in $\RR^d$ by $\dotp{\cdot}{\cdot}$ and its induced $L_2$-norm by $\|\cdot\|$. The ceiling function is denoted by $\lceil\cdot\rceil$. We also write $\updelta_x$ a Dirac measure at the point $x\in\RR^d$. 	

\subsection{Problem Formulation}    
We consider a distributed training setting with $M$ workers, each with a possibly heterogeneous (i.e. non-\iid) underlying true data distribution $\Prob_m$ ($m\in\setM$). The true data distribution $\Prob_m$ is usually unknown, but can be approximated by its empirical distribution $\Probhat_m = \frac{1}{n_m}\sum_{i=1}^{n_m}\updelta_{\xi_i^m}$, where $\calD_m \coloneqq\{\xi_i^m\}_{i\in\set{n_m}}$ is the set of $n_m$ data samples of worker $m$ and $\xi_i^m\in\calZ_m\subseteq\RR^p$ for $i\in\set{n_m}$. The goal is to find an approximate minimizer of the global objective $F$ which is the average of the local objectives $F_m$: 
\begin{equation*}%\label{eqn:fl}
    \minimize_{x\in\RR^d} \ F(x) \coloneqq \frac1M\sumM F_m(x), 
\end{equation*}
where $F_m(x) \coloneqq \Ex_{\xi\sim\Probhat_m}[f_m(x;\xi)] = \frac{1}{n_m}\sum_{i=1}^{n_m} f_m(x; \xi_i^m)$ is the local objective for worker $m$. 

At each iteration $k\in\NN$, since the number of samples of each worker $n_m$ could be large, the per-worker gradient $\nabla F_m$ cannot be computed. It can however be approximated by its minibatch counterpart 
\begin{equation}\label{eqn:batch_grad}
    (\forall x\in\RR^d)\quad\nabla F_{\calB_k^m}(x) \coloneqq \frac{1}{b_k^m}\sum_{i\in\calB_k^m} \nabla f_m(x; \xi_i^m),
\end{equation}
where $\calB_k^m$ is the \emph{local batch} of worker $m$ at iteration $k$ and $b_k^m\coloneqq|\calB_k^m|$ is the corresponding \emph{local batch size}, assuming that $f_m(\cdot; \xi^m)$ is continuously differentiable for any $\xi^m\in\calZ_m$.

\subsection{Minibatch and Local  Stochastic Gradient Methods}
We present both the formulations of minibatch \SGD under data parallelism and Local \SGD. 

\paragraph{Minibatch \SGD.}
Minibatch \SGD under data parallelism is a fully synchronized stochastic gradient method. In particular, at each iteration $k$, after each worker $m$ has computed its local batch gradient $\nabla F_{\calB_k^m}(x_k)$, the global batch gradient $\nabla F_{\calB_k}(x_k)$ is computed by averaging all local batch gradients by $\nabla F_{\calB_k}(x_k) = \frac1M\sumM \nabla F_{\calB_k^m}(x_k)$ with a global batch size $b_k\coloneqq|\calB_k| = \sumM b_k^m = Mb_k^1$, assuming that all local batch sizes are equal. This is followed by the global model update, performed via $x_{k+1} = x_k - \alpha_k \nabla F_{\calB_k}(x_k)$, where $\alpha_k$ is the learning rate of iteration $k$. The same mechanism also applies to other optimizers such as \Adam \citep{kingma2015} and \AdaGrad \citep{duchi2011adagrad}.

\paragraph{Local \SGD.}     
As opposed to minibatch \SGD, Local \SGD \citep{stich2019local} reduces the communication frequency. Communication for \emph{model averaging} (through \emph{all-reduce} operations) is performed every $H$ local gradient steps, Local \SGD takes the following updates: for $m\in\setM)$ and $k\in\NN$, 
\begin{equation}\label{eqn:local_sgd}
    x_{k+1}^m = 
    \begin{cases}
        \frac1M\sumM \left(x_k^m - \alpha_k \nabla F_{\calB_k^m}(x_k^m)\right) & \text{if } H | k+1, \\
        x_k^m - \alpha_k \nabla F_{\calB_k^m}(x_k^m) & \text{otherwise}.
    \end{cases}
\end{equation}    
We also consider the following alternative representation that explicitly separates the numbers of local gradient steps $H$ and communication rounds $K$. In particular, for $M$ parallel workers, $K$ rounds of communication, $H$ local gradient steps per round, Local \SGD updates can also be expressed as: for $m\in\setM)$ and $k\in\set{0, K-1}$,      
\begin{equation}\label{eqn:local_sgd_alt}        
    % (\forall m\in\setM)(\forall k\in\set{0, K-1})\;
    \begin{cases}
        x_{k,h+1}^m = x_{k,h}^m - \alpha_{k,h} \nabla F_{\calB_{k,h}^m}(x_{k,h}^m) & \hspace*{-2mm}\text{for }h\in\set{0,H-1}, \\
        x_{k+1,0}^m = \frac1M\sumM x_{k,H}^m.
    \end{cases}
\end{equation}

\section{Adaptive Local Batch Size Strategies for Local Gradient Methods}
The proposed adaptive local batch size strategies are based on the adaptive sampling method \citep{byrd2012sample} originally developed for unconstrained stochastic optimization problems using minibatch \SGD. We outline its mechanism and its extension to Local \SGD below. 

\subsection{Adaptive Sampling Methods}
Introduced in \citet{byrd2012sample}, the adaptive sampling method for minibatch \SGD was designed for a single worker ($M=1$) with the objective $F(x) = \frac{1}{n}\sum_{i=1}^n f(x; \xi_i)$ over samples $\calD=\{\xi_i\}_{i=1}^n$. This method relies on a key stochastic optimization principle: when the sample size $n$ is large, the batch gradient $\nabla F_{\calB_k}$ approximates the full gradient $\nabla F$. If $f(\cdot;\xi)$ is continuously differentiable and convex for any $\xi \in \calZ \subseteq \RR^p$, then $-\nabla F_{\calB_k}(x_k)$ serves as a descent direction for $F$ at $x_k \in \RR^d$, provided there exists $\eta\in\left(0,1\right)$ such that
\begin{equation}\label{eqn:condition}
    \delta_{\calB_k}(x_k) \coloneqq \|\nabla F_{\calB_k}(x_k) - \nabla F(x_k)\| \le \eta\|\nabla F(x_k)\|. 
\end{equation}
Note that $\nabla F(x_k)$ is not available in minibatch \SGD, but $\delta_{\calB_k}(x_k)$ can be approximated by $\hat\delta_{\calB_k}(x_k)^2\coloneqq\frac{1}{b_k}\Var_{i\in\calB_k}(\nabla f(x_k; \xi_i))\cdot\frac{n-b_k}{n-1}$, where, for any vector-valued function $h\colon\RR^d\times\calZ\to\RR^d$, the batch variance is defined by
\begin{equation}\label{eqn:def_var}
    \Var_{i\in\calB_k}(h(x_k; \xi_i)) 
    \coloneqq \frac{1}{b_k-1}\sum_{i\in\calB_k} \norm{h(x_k; \xi_i) - \Ex_{i\in\calB_k}[h(x_k; \xi_i)]}^2, 
\end{equation}
with $\Ex_{i\in\calB_k}[h(x_k; \xi_i)] \coloneqq\frac{1}{b_k}\sum_{i\in\calB_k} h(x_k; \xi_i)$. Hence, as $n\to\infty$, condition \eqref{eqn:condition} can be approximated as 
\begin{equation}\label{eqn:approx_norm}
    \frac{1}{b_k}\Var_{i\in\calB_k}(\nabla f(x_k; \xi_i))\le \eta^2 \|\nabla F_{\calB_k}(x_k)\|^2. 
\end{equation}

In practice, for each iteration $k\in\NN$, the \emph{dynamic sample gradient algorithm} \citep{byrd2012sample} performs the following minibatch \SGD update: $x_{k+1} = x_k - \alpha_k \nabla F_{\calB_k}(x_k)$ with learning rate $\alpha_k>0$, and then checks the following \emph{(approximate) norm test} condition \eqref{eqn:approx_norm}, which can also be viewed as an approximation of the \emph{exact variance norm test} in the stochastic setting:
\begin{equation}\label{eqn:exact_norm}
    \Ex_k\left[\|\nabla F_{\calB_k}(x_k) - \nabla F(x_k) \|^2\right] \le \eta^2 \|\nabla F(x_k)\|^2, 
\end{equation}
i.e., condition \eqref{eqn:condition} holds in expectation. Here we abbreviate the conditional expectation on $\calF_k$ (i.e., the filtration $\calF_k\coloneqq\sigma(\{x_0, \calB_0, \calB_1, \ldots, \calB_{k-1}\})$) by $\Ex_k[\cdot]\coloneqq\Ex[\cdot\mid\calF_k]$. 
If condition \eqref{eqn:approx_norm} is violated, the batch size of the next iteration $b_{k+1}$ is increased, determined by: 
\begin{equation}\label{eqn:norm_test}
    b_{k+1} =\left \lceil \frac{\Var_{i\in\calB_k}(\nabla f(x_k; \xi_i))}{\eta^2  \|\nabla F_{\calB_k}(x_k)\|^2} \right\rceil, 
\end{equation}
and the next batch $\calB_{k+1}$ is sampled with this larger batch size accordingly. Otherwise, the next batch size remains the same as the current one. 
The value of the pre-specified constant $\eta$ controls the probability of obtaining a descent direction, thereby affecting the probability of increasing the batch size as well as the magnitude of the next batch size.

Condition \eqref{eqn:exact_norm} is more relaxed than the typical assumption of uniformly bounded gradient variance for minibatch stochastic gradient methods: for any \iid~batch $\calB\subset\setn$ and some positive constant $\sigma>0$, $\Ex_{\calB}\left[\|\nabla F_{\calB}(x) - \nabla F(x)\|^2\right] \le \sigma^2$ for all $x\in\RR^d$, which is usually unverifiable in practice. The adaptive batch size schedule using the approximate norm test may cause rapid batch size increases, undermining its purpose. To moderate these schedules, \citet{bollapragada2018adaptive} introduced the (augmented) inner product test, which manages the variance between the batch and full gradient's inner product. Given the preference for large batch sizes for memory efficiency, controlled via $\\eta$, we focus on the norm test and defer exploring local variants of the (augmented) inner product test to future work.

\subsection{Extensions to Local Gradient Methods}
We extend the above adaptive sampling methods beyond the single worker setting with data parallelism using multiple workers, more specifically for Local \SGD.     
We have to invoke a local variant of the norm test \eqref{eqn:exact_norm}, called the \emph{(exact variance) local norm test}, at every iteration $k\in\NN$, using formulation \eqref{eqn:local_sgd}: for some sequence of positive constants $(\eta_m)_{m\in\setM}\subset(0,1)$, for all $m\in\setM$, 
\begin{equation}\label{eqn:exact_norm_local}
    \Ex_{\calF_k^m}\left[\| \nabla F_{\calB_k^m}(x_k^m) - \nabla F_m(x_k^m) \|^2\right] \le \eta_m^2\norm{\nabla F_m(x_k^m)}^2,  
\end{equation}    
where $\Ex_{\calF_k^m}[\cdot]\coloneqq\Ex[\cdot\mid\calF_k^m]$ denotes the conditional expectation on $\calF_k^m\coloneqq\sigma(\{x_0, \calB_0^m, \ldots, \calB_{k-1}^m\})$. 
Again, $\nabla F_m$ is in general unavailable since the sample size $n_m$ could be large. The above local norm test is approximately enforced by the following (approximate) \emph{local norm test}: for all $m\in\setM$, 
\begin{equation}\label{eqn:approx_norm_local}
    \frac{1}{b_k^m}\Var_{i\in\calB_k^m}\left(\nabla f_m(x_k^m; \xi_i^m) \right) \le \eta_m^2\norm{\nabla F_{\calB_k^m}(x_k^m)}^2. 
\end{equation}
Should the condition \eqref{eqn:approx_norm_local} not be met, the subsequent local batch sizes are determined by the formula: 
\begin{equation}\label{eqn:test_stat}
    (\forall m\in\setM)\quad b_{k+1}^m =\left \lceil \frac{\Var_{i\in\calB_k^m}(\nabla f_m(x_k^m; \xi_i^m))}{\eta_m^2  \|\nabla F_{\calB_k^m}(x_k^m)\|^2} \right\rceil, 
\end{equation}
and the local batches $\calB_{k+1}^m$ are augmented accordingly.     
In a practical implementation, for simplicity, we can set $\eta_m \equiv \eta$ for all $m\in\setM$ and choose the next local batch sizes to be, for all $m\in\setM$,  
\begin{equation*}
    b_{k+1}^m \coloneqq \max_{m\in\setM} \ \left \lceil \frac{\Var_{i\in\calB_k^m}(\nabla f_m(x_k^m; \xi_i^m))}{\eta^2  \|\nabla F_{\calB_k^m}(x_k^m)\|^2} \right\rceil, 
\end{equation*}
which is particularly useful when homogeneous workers and identical data on the workers are used, so that there would not be discrepancy of training time for the local steps due to different local batch sizes and undesirable stragglers before model averaging. 

The proposed adaptive batch size strategies extend to local variants of minibatch stochastic gradient optimizers beyond \SGD \citep{robbins1951}, including momentum \SGD \citep{sutskever2013importance}, \AdaGrad \citep{duchi2011adagrad}, \Adam \citep{kingma2015}, and \AdamW \citep{loshchilov2019decoupled}. 

In a data-parallel setting, adaptive global batch size strategies can be developed using minibatch \SGD, leading to increased communication overheads from aggregating local batch gradient variances among workers to compute the quantities in \eqref{eqn:approx_norm}.

\subsection{Practical Considerations}
The (approximate) norm test \eqref{eqn:approx_norm} and its local variant \eqref{eqn:approx_norm_local} extend the norm test \citep{byrd2012sample} for Local \SGD, but implementing them in deep learning libraries like PyTorch \citep{paszke2019pytorch} is complex, as only batch gradients $\nabla F_{\calB_k^m}(x_k^m)$, not per-sample gradients $\nabla f_m(x_k^m; \xi_i^m)$, are available in the backward training step. Therefore, we cannot compute the variance of per-sample gradients without additional subroutines (e.g., using \texttt{torch.func}). These methods do not scale well to large models, as they require a full model replica and per-sample gradients matching the model's parameter size. Increasing batch sizes during training elevates memory demands, necessitating many gradient accumulation steps (a serial computation) to simulate large batches. We, therefore, suggest an alternative to approximate the variance term in the (local) norm test.

We can use local batch gradients as estimators instead of per-sample gradients. To illustrate, consider minibatch \SGD with data parallelism among multiple workers where $ \nabla F_{\calB_k}(x_k) = \frac1M\sumM\nabla F_{\calB_k^m}(x_k)$ and $x_k^m = x_k$ apply for $m\in\setM$. Note that 
$\Var_{m\in\setM}\left( \nabla F_{\calB_k^m}(x_k)\right) \coloneqq \frac1{M-1}\sumM \norm{\nabla F_{\calB_k^m}(x_k) - \nabla F_{\calB_k}(x_k) }^2$. 
Then, by the definition of $\nabla F_{\calB_k^m}(x_k)$, we can write 
\begin{equation*}
    \Var_{m\in\setM}\left( \nabla F_{\calB_k^m}(x_k)\right) 
    = \sum_{i\in\calB_k^m} \frac{\Var_{i\in\calB_k^m}\left(\nabla f(x_k; \xi_i^m)\right)}{(b_k^m)^2} 
    = \frac{1}{b_k^m}\Var_{i\in\calB_k}\left(\nabla f(x_k; \xi_i)\right) = \frac{M}{b_k}\Var_{i\in\calB_k}\left(\nabla f(x_k; \xi_i)\right), 
\end{equation*}
since $b_k^m = b_k/M$. Therefore, for the norm test \eqref{eqn:norm_test}, we can instead compute
$\Var_{i\in\calB_k}\left(\nabla f(x_k; \xi_i)\right) 
= b_k/M\cdot \frac{1}{M-1}\sumM \norm{\nabla F_{\calB_k^m}(x_k) - \nabla F_{\calB_k}(x_k) }^2$. 

We extend this workaround to Local \SGD, focusing on the \iid~setting where $\Probhat_m \equiv \Probhat$ and $f_m\equiv f$, thus resulting in $F_m=F$ for all $m\in\setM$. Varying $x_k^m$'s across workers $m\in\setM$ due to different $\calB_k^m$'s prevent local gradients from being evaluated at the same iterates as in minibatch \SGD during updates. Thus, we instead compute 
\begin{equation*}
    \Var_{i\in\calB_k}\left(\nabla f(x_k^m; \xi_i^m)\right) 
    = \frac{b_k}{M}\cdot \frac1{M-1}\sumM \norm{\nabla F_{\calB_k^m}(x_k^m) - \gbark }^2, 
\end{equation*}
where $\gbark \coloneqq\frac1M\sumM \nabla F_{\calB_k^m}(x_k^m)$. 

The actual \emph{(exact variance) norm test for local gradient methods} is given by, for all $m\in\setM$, 
\begin{equation}\label{eqn:exact_norm_actual}
    \Ex_{\calF_k^m}\left[\| \nabla F_{\calB_k^m}(x_k^m) - \nabla F(x_k^m) \|^2\right] \le \eta_m^2\norm{\nabla F(x_k^m)}^2,  
\end{equation}    
In our implementation, the above actual \emph{(exact variance) norm test for local gradient methods} is approximately enforced by the following (approximate) \emph{norm test} for local gradient methods: 
\begin{equation}\label{eqn:approx_norm_local_actual}
    \frac{1}{b_k}\Var_{i\in\calB_k}\left(\nabla f(x_k^m; \xi_i^m) \right) \le \eta^2\norm{\gbar_k}^2. 
\end{equation}
Should the condition \eqref{eqn:approx_norm_local_actual} not be met, the subsequent local batch sizes are determined by the formula: 
\begin{equation}\label{eqn:test_stat_actual}
    (\forall m\in\setM)\; b_{k+1}^m =\left \lceil \frac{\Var_{i\in\calB_k}\left(\nabla f(x_k^m; \xi_i^m) \right) }{M\eta^2 \norm{\gbar_k}^2} \right\rceil, 
\end{equation}
and the local batches $\calB_{k+1}^m$ are augmented accordingly. Note that the computation of $\gbar_k$ requires additional all-reduce communication of $g_k^m$ among the $M$ workers. As a result, to avoid unnecessary additional communication rounds due to the test, we only perform the test every $H$ local gradients steps, i.e., performing model parameter and gradient synchronization at the same time and frequency. The algorithm pseudocode is detailed in \Cref{sec:algorithm}.

% \vspace*{-2.5mm}
\section{Convergence Analysis}
\label{sec:conv}

Our analysis focuses on the \iid~setting where $\Probhat_m \equiv \Probhat$ and $f_m\equiv f$, and hence $F_m=F$ for all $m\in\setM$. We set $\eta_m \equiv \eta \in (0,1)$ for all $m \in \setM$, typical in datacenter environments for distributed training \citep{dean2012large}, where workers access the entire data set. The analysis of heterogeneous data settings is deferred to future work. Our convergence analysis, following \citet{stich2020error}, does not assume bounded gradient variance across $\RR^d$. Instead, our adaptive local batch size strategies ensure the \emph{expected strong growth} condition \citep{khaled2023better,bottou2018optimization} at each local iterate $x_k^m$ for all workers $m \in \setM$. This approach relaxes the standard assumptions of uniform boundedness and expected strong growth conditions \citep{vaswani2019fast}, which are challenging to verify except in the interpolation regime \citep{ma2018power}.

%\subsection{Assumptions}

We might need to invoke the following assumptions on the objective function $F$. 
\begin{assumption}[$L$-Lipschitz smoothness]\label{ass:smooth}
    The function $F\colon\RR^d\to\RR$ is continuously differentiable, bounded below by $F^\star\coloneqq\inf_{x\in\RR^d} F(x)=F(x^\star)\in\RR$ with $x^\star=\argmin_{x\in\RR^d} F(x)$, and is $L$-Lipschitz smooth for some $L>0$, i.e., for any $(x,y)\in\RR^d\times\RR^d$, $\norm{\nabla F(x) - \nabla F(y)} \le L\norm{x - y}$.         
\end{assumption}    

\begin{assumption}[$\mu$-strong convexity]\label{ass:strong_cvx}
    The function $F$ is $\mu$-strongly convex for some $\mu\ge0$, i.e., for any $(x,y)\in\RR^d\times\RR^d$, $F(x) - F(y) + \frac\mu2\norm{x-y}^2 \le \dotp{\nabla F(x)}{x-y}$. If $\mu=0$, we say that $F$ is convex.         
\end{assumption}

% \subsection{Convergence of Adaptive Local Batch Size Strategies for Local \SGD}

We derive the following convergence results for the proposed adaptive batch size strategies as follows. Their proofs are deferred to \Cref{sec:proofs}. 

\begin{theorem}[Strongly convex; $\mu>0$]\label{thm:conv_strong_cvx}
    Suppose that \Cref{ass:smooth,ass:strong_cvx} hold with $\mu>0$. 
    Let $(x_k^m)_{k\in\NN, m\in\setM}$ be the sequence of the iterates of Local \SGD \eqref{eqn:local_sgd} with the (exact variance) local norm test \eqref{eqn:exact_norm_local} with $\eta_m\equiv\eta\in(0,1)$ and a constant learning rate $\alpha_k\equiv\alpha$. Then, for some $\alpha\le 1/(10L(HM+\eta^2))$, we have
    \begin{equation}
        \Ex[F(x_\out)] - F^\star  
        = \tilde\scrO\left(L(HM+\eta^2)\e^{-\frac{\mu K}{10L(HM+\eta^2)}} \norm{x_0-x^\star}^2\right), 
    \end{equation}
    where $x_\out\in\{x_k^m\}_{k\in\set{0,K-1}, m\in\setM}$ is chosen as $x_k^m$ with probability proportional to $(1-\mu\alpha/2)^{-k}$ uniformly on $m\in\setM$ and $\tilde{\scrO}$ hides logarithmic factors. 
\end{theorem}

\begin{theorem}[Convex; $\mu=0$]\label{thm:conv_cvx}
    Suppose that \Cref{ass:smooth,ass:strong_cvx} hold with $\mu=0$. 
    Let $(x_k^m)_{k\in\NN, m\in\setM}$ be the sequence of the iterates of Local \SGD \eqref{eqn:local_sgd} with the (exact variance) local norm test \eqref{eqn:exact_norm_local} with $\eta_m\equiv\eta\in(0,1)$ and a constant learning rate $\alpha_k\equiv\alpha$. Then, for some $\alpha\le 1/(10L(HM+\eta^2))$, we have
    \begin{equation}
        \Ex[F(x_\out)] - F^\star = \scrO\left( \frac{L(HM+\eta^2)}{K}\norm{x_0-x^\star}^2\right), 
    \end{equation}
    where $x_\out$ is chosen uniformly at random from $\{x_k^m\}_{k\in\set{0,K-1}, m\in\setM}$.
\end{theorem}    

\begin{theorem}[Nonconvex]\label{thm:conv_noncvx}
    Suppose that \Cref{ass:smooth} holds. 
    Let $(x_k^m)_{k\in\NN, m\in\setM}$ be the sequence of the iterates of Local \SGD \eqref{eqn:local_sgd} with the (exact variance) local norm test \eqref{eqn:exact_norm_local} with $\eta_m\equiv\eta\in(0,1)$ and a constant learning rate $\alpha_k\equiv\alpha$. Then, for some $\alpha\le 1/(10L(HM+\eta^2))$, we have
    \begin{equation}
        \Ex\norm{\nabla F(x_\out)}^2 = \scrO\left(\frac{L(HM+\eta^2)}{K}(F(x_0) - F^\star) \right), 
    \end{equation}
    where $x_\out$ is chosen uniformly at random from $\{x_k^m\}_{k\in\set{0,K-1}, m\in\setM}$.
\end{theorem}

The theorems confirm that Local \SGD, using the local norm test with exact variance, achieves linear convergence rates for both convex and nonconvex objectives. This scaling is proportional to the number of local gradient steps $H$ and workers $M$, akin to Local \SGD under constant batch sizes and bounded gradient variance \citep{stich2020error}.

\begin{remark}[Choice of $\eta$]
    Convergence rates depend quadratically on $\eta$, which controls the likelihood and size of increased local batch sizes. Choosing smaller $\eta$ values accelerates convergence but causes quick batch size escalation, defeating the original purpose of adaptive batch sizing and widening the generalization gap.
\end{remark}

\section{Numerical Experiments}
\label{sec:expt}
We evaluate the efficiency of the training (total steps or wall-clock time) and the generalization performance of our strategies through two standard deep learning applications in two domains---image classification and language modeling tasks. We use a \ResNet-50 \citep{he2016deep} on the CIFAR-10 dataset \citep{krizhevsky2009learning} and a \ResNet-101 on the ImageNet dataset \citep[see \Cref{subsec:resnet-101_imagenet_supp}]{russakovsky2015imagenet} for image classification, and a MicroLlama 300M \citep{microllama2024} on the C4 dataset \citep{raffel2020exploring} for language modeling, with four workers ($M=4$ GPUs). The models are trained using Local \AdamW and Local \SGD with momentum (a.k.a.~stochastic heavy ball; \SHB), synchronizing only the models every $H$ local steps. Further experimental details and results are given in \Cref{sec:expt_details}. When the chosen batch sizes cannot fit in the GPU memory, we use the technique of gradient accumulation.

\subsection{\ResNet-50 on CIFAR-10}
\label{sec:resnet50_on_cifar10}
We fit \ResNet-50 for CIFAR-10 using $M=4$ workers, training over 30 million samples (600 epochs) with global and local batch sizes of 50,000 and 12,500. Momentum \SGD \citep[a.k.a.~stochastic heavy ball]{sutskever2013importance} serves as the inner optimizer. Learning rates are managed with a 10\% linear warmup and cosine decay, peaking at 0.05 and bottoming out at 0.005. Additionally, we apply linear scaling of learning rates \citep{krizhevsky2014one,goyal2017accurate} relative to batch sizes for constant batch size baselines.

\begin{figure*}[t]    
    \centering
    \begin{subfigure}[h]{0.49\textwidth}
        \centering
        \includegraphics[width=\textwidth]{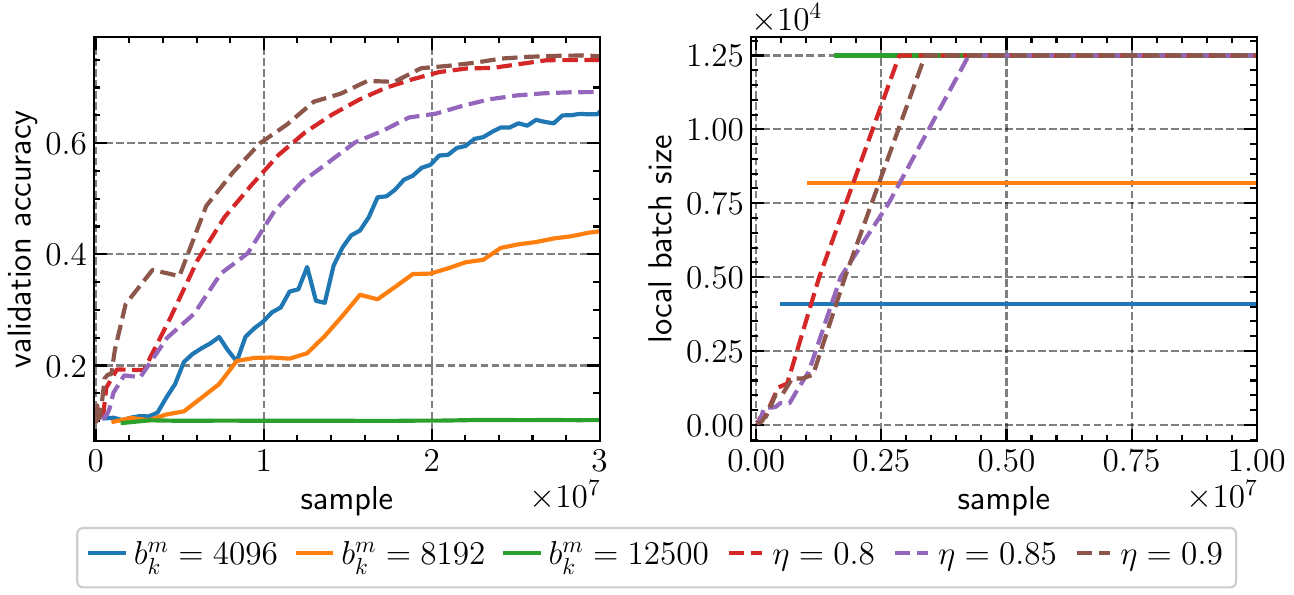}
        \caption{$H=32$ with $b_k^m\in\{4096, 8192, 12500\}$ and $\eta\in\{0.8,0.85,0.9\}$}
        \label{fig:resnet-50_cifar10_const_H32}
    \end{subfigure}%
    \hfill
    \begin{subfigure}[h]{0.49\textwidth}
        \centering
        \includegraphics[width=\textwidth]{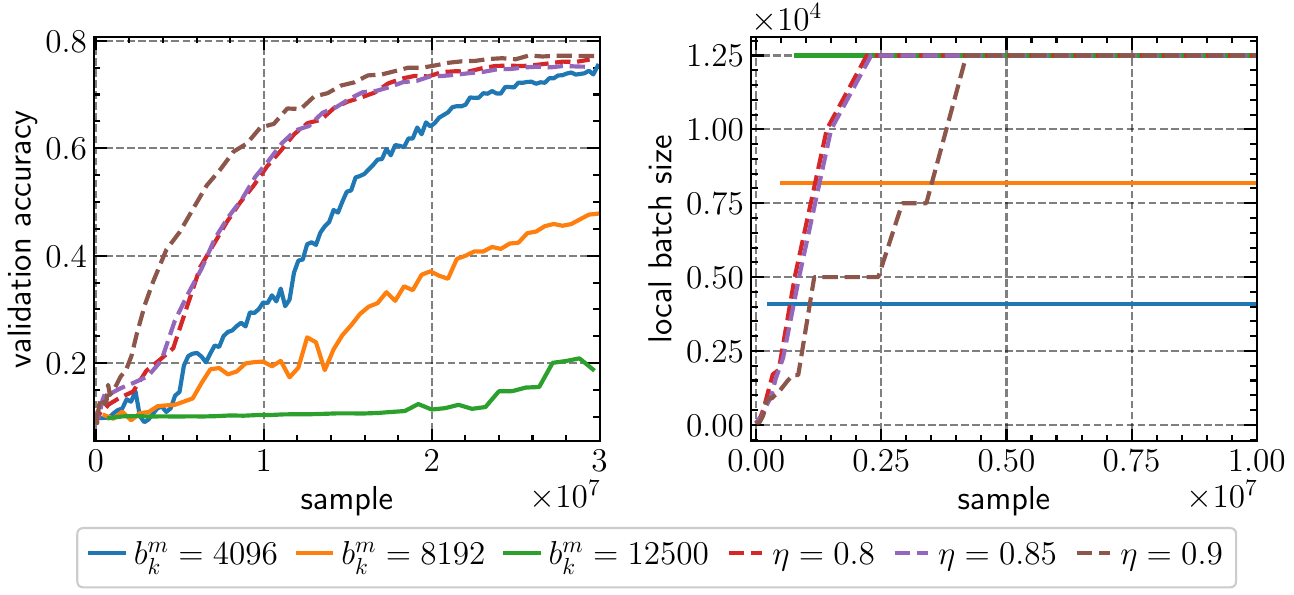}
        \caption{$H=16$ with $b_k^m\in\{4096, 8192, 12500\}$ and $\eta\in\{0.8,0.85,0.9\}$}
        \label{fig:resnet-50_cifar10_const_H16}
    \end{subfigure}%
    \hfill
    \begin{subfigure}[h]{0.49\textwidth}
        \centering
        \includegraphics[width=\textwidth]{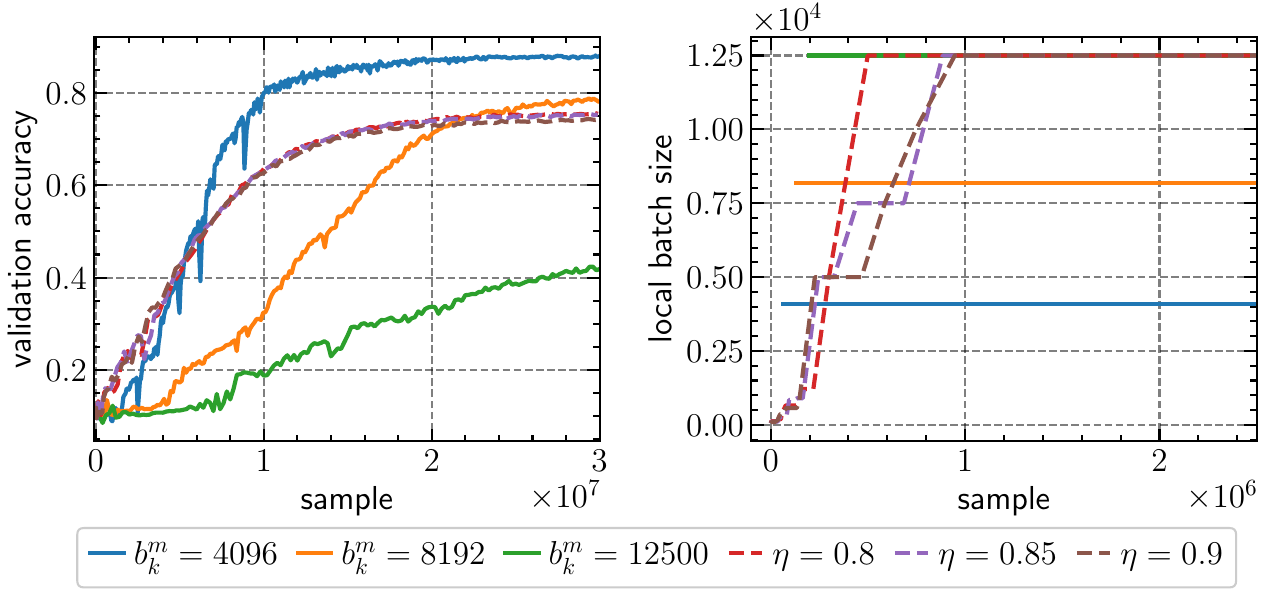}
        \caption{$H=4$ with $b_k^m\in\{4096, 8192, 12500\}$ and $\eta\in\{0.8,0.85,0.9\}$}
        \label{fig:resnet-50_cifar10_const_H4}
    \end{subfigure}%
    \hfill
    \begin{subfigure}[h]{0.49\textwidth}
        \centering
        \includegraphics[width=\textwidth]{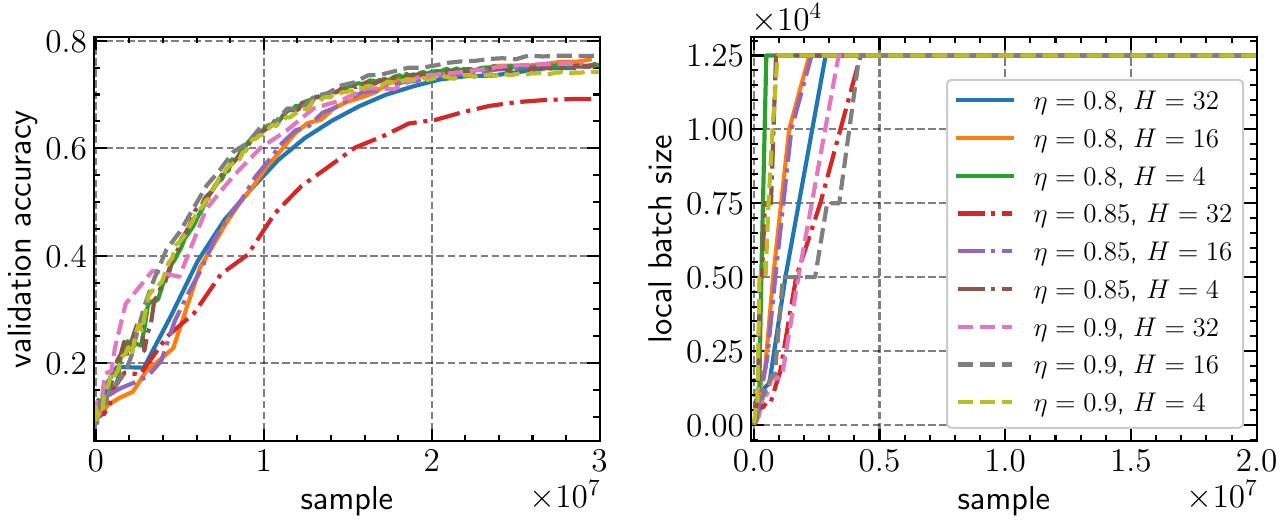}
        \caption{Adaptive batch size strategies for $H\in\{32, 16, 4\}$}
        \label{fig:resnet-50_cifar10_final}
    \end{subfigure}
    \caption{Validation accuracy and local batch sizes of Local \SHB with adaptive batch size strategies for \ResNet-50 on CIFAR-10. }
    \label{fig:resnet-50_cifar10}
%    \vspace*{-4mm}
\end{figure*}

%\vspace*{-4mm}
\paragraph{Values of $\eta$.} \Cref{table:resnet-50_cifar10} demonstrates that adaptive local batch sizes offer better validation accuracy than constant batch sizes with similar averaged batch size. For example, using $H=16$ and $\eta = 0.8$, the adaptive approach averages a batch size of 8906, more than double a constant size of 4096, achieving over 1\% higher accuracy and nearly 1000 fewer steps. This highlights the effectiveness of adaptive strategies in reducing the generalization gap in large-batch training, though it requires 16\% more training time due to extra computations from the norm test. Similar results are observed with other local gradient step values $H$.

%\vspace*{-4mm}
\paragraph{Local gradient steps $H$.}
We examine the impact of local gradient steps $H$ (\Cref{table:resnet-50_cifar10}, \Cref{fig:resnet-50_cifar10}). According to \citet{lin2020dont,ortiz2021trade,gu2023why,gu2024quadratic}, local gradient methods may enhance generalization compared to synchronized ones $H=1$, although we use \SHB instead of vanilla \SGD in synchronized models. This is evident with a constant batch size of 8192 and 4 local gradient steps. For adaptive batch sizes, smaller $H$ values rapidly increase batch sizes (\Cref{fig:resnet-50_cifar10}), supporting our theory that a smaller $\eta$ with a larger $H$ ensures convergence rates of $\scrO(HM+\eta^2)$. See \Cref{fig:resnet-50_cifar10} and \Cref{subsec:resnet-50_cifar10_supp} for more details.

\begin{table*}[h!]
    % \vspace*{-4mm}
    \centering
    \caption{Adaptive batch size strategies for \ResNet-50 on CIFAR-10; no.~of total steps (steps), wall-clock time (in hours), average local batch sizes (bsz.), and best validation accuracy (acc.; in \%)}
    \label{table:resnet-50_cifar10}
    \vspace*{-2mm}
    \tiny
    \renewcommand{\arraystretch}{1.}
    \begin{tabular}{l|rrrr|rrrr|rrrr|rrrr}
        \toprule[1pt]
        \multirow{2}{*}{Schedule}
        & \multicolumn{4}{c|}{$H=32$}
        & \multicolumn{4}{c|}{$H=16$}
        & \multicolumn{4}{c|}{$H=4$}
        & \multicolumn{4}{c}{$H=1$}
        \\
        & steps  & time & bsz. & acc. & steps  & time & bsz. & acc. & steps  & time & bsz. & acc. & steps & time & bsz. & acc. \\
        \midrule
        Constant & 1824 & 0.98 & 4096 & 67.02 & 1824 &  0.99 & 4096 & 75.32  & 1828 &  1.07 & 4096 & 88.12 & 1831 &  1.34 & 4096 & 89.40 \\ 
        Constant & 896 & 0.95 & 8192 & 44.27 & 912 &  0.98 & 8192 & 48.19  & 912  &  1.01 & 8192 & 78.81 & 915  &  1.15 & 8192 & 76.58 \\ 
        Constant & 576 & 1.07 & 12500 & 10.19 & 592  &  1.10 & 12500 & 20.89 & 596  &  1.13 & 12500 & 42.36 & 599  &  1.23 & 12500 & 53.80 \\ 
        $\eta=0.8$ &  928 & 1.13 & 7828  &  74.95   & 832  &  1.15 & 8906 & 76.50 & 744  &  1.16 & 10060 & 75.67 & 1241 &  1.41 & 6043 & 82.14 \\ 
        $\eta=0.85$ &  1088 & 1.18 & 7019  & 69.92  & 864  &  1.14 & 8607 & 75.32 & 756  &  1.16 & 9896 & 75.40 & 1270 &  1.43 & 5906 & 83.15 \\ 
        $\eta=0.9$ & 1216 & 1.15 & 6125  & 75.76  & 1088 &  1.16 & 6929 & 77.48 & 748  &  1.17 & 10022 & 74.35 & 1540 &  1.47 & 4868 & 84.61 \\ 
        \bottomrule
    \end{tabular}
    \renewcommand{\arraystretch}{1}
    % \vspace{-2.5mm}
\end{table*}

\subsection{MicroLlama 300M on C4}

We pre-train MicroLlama 300M \citep{microllama2024} using the C4 dataset \citep{raffel2020exploring} processed by Allen AI\footnote{See \url{https://huggingface.co/datasets/allenai/c4}}, tokenized with Llama 2 7B's tokenizer \citep{touvron2023llama2} (vocabulary size: 32,000). Employing $M=4$ workers, the training budget is 4 billion tokens with a sequence length of 2048, and maximum batch sizes of 8192 (global) and 2048 (local) sequences, equating to $\approx$ 16M and 4M tokens. We use \AdamW \citep{loshchilov2019decoupled} as the inner optimizer, with a learning rate schedule featuring a 1\% linear warmup and cosine decay, peaking at 0.001 and dropping to 0.0001.

\begin{table*}[h!]
    \centering
    \caption{Adaptive batch size strategies for MicroLlama 300M on C4; no.~of total steps (steps), wall-clock time (in hours), average local batch sizes (bsz.), and best validation loss (cross-entropy loss; estimated by 100 iterations)}
    \label{table:microllama}
    \vspace*{-2mm}
    \scriptsize
    \renewcommand{\arraystretch}{1.}
    \begin{tabular}{l|rrrr|rrrr|rrrr}
        \toprule[1pt]
        \multirow{2}{*}{Schedule}
        & \multicolumn{4}{c|}{$H=32$}
        & \multicolumn{4}{c|}{$H=16$}
        & \multicolumn{4}{c}{$H=4$}
        \\
        & steps & time & bsz. & loss & steps & time & bsz. & loss & steps  & time & bsz. & loss \\
        \midrule
        Constant & 31744 & 10.59 & 512 & 4.10 & 15616 & 6.86 & 512 & 4.20 & 3888 & 11.91 & 512 & 3.93  \\
        Constant & 16384 & 10.53 & 1024 & 4.82 & 7936 & 10.64 & 1024 & 4.84 & 1968 & 11.31 & 1024 & 5.02 \\
        Constant & 8192 & 9.77 & 2048 & 5.72 & 4096 & 10.50 & 2048 & 5.73 & 992 & 10.96 & 2048 & 6.00 \\
        $\eta=0.8$ & 15360 & 11.13 & 1088 & 4.55 & 5632 & 10.96 & 1453 & 4.98 & 1216 & 11.13 & 1658 & 5.05  \\
        $\eta=0.9$ & 16384 & 11.54 & 1054 & 4.66 & 6400 & 11.22 & 1299 & 4.80 & 1360 & 11.18 & 1484 & 4.68 \\
        \bottomrule
    \end{tabular}
    \renewcommand{\arraystretch}{1}
%    \vspace*{-2.5mm}
\end{table*}

%\vspace*{-2mm}
\paragraph{Values of $\eta$.} 
Similar to the observation in \Cref{sec:resnet50_on_cifar10}, the adaptive batch size strategy generally achieves lower losses compared to their constant batch size counterparts, with a significant reduction in the number of training steps. It illustrates the effectiveness of our adaptive batch size strategy in bridging the generalization gap in large-batch training.

%\vspace*{-2mm}
\paragraph{Local gradient steps $H$.}
Again, batch sizes grows more rapidly as the number of local gradient steps $H$ decreases, indicating an inverse scaling relationship between $\eta$ and $H$. This in turn leads to smaller batch sizes at the earlier stages of training and hence better generalization. 
We also observe from \Cref{fig:microllama_final} that, for $\eta=0.9$, the generalization gaps between different values of $H$ is narrower than that for $\eta=0.8$.

% \newpage
\begin{figure*}[h!]    
    \centering
    \begin{subfigure}[h]{0.49\textwidth}
        \centering
        \includegraphics[width=\textwidth]{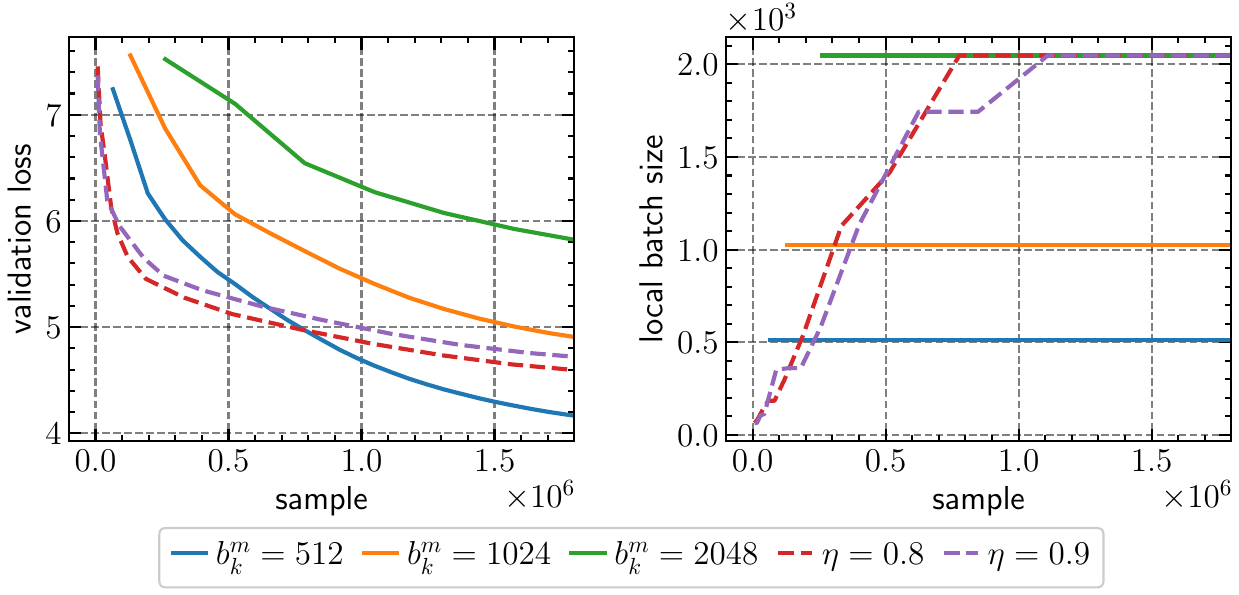}
        \caption{$H=32$ with $b_k^m\in\{512, 1024, 2048\}$ and $\eta\in\{0.8,0.9\}$}
        \label{fig:microllama_const_H32}
    \end{subfigure}%
    \hfill
    \begin{subfigure}[h]{0.49\textwidth}
        \centering
        \includegraphics[width=\textwidth]{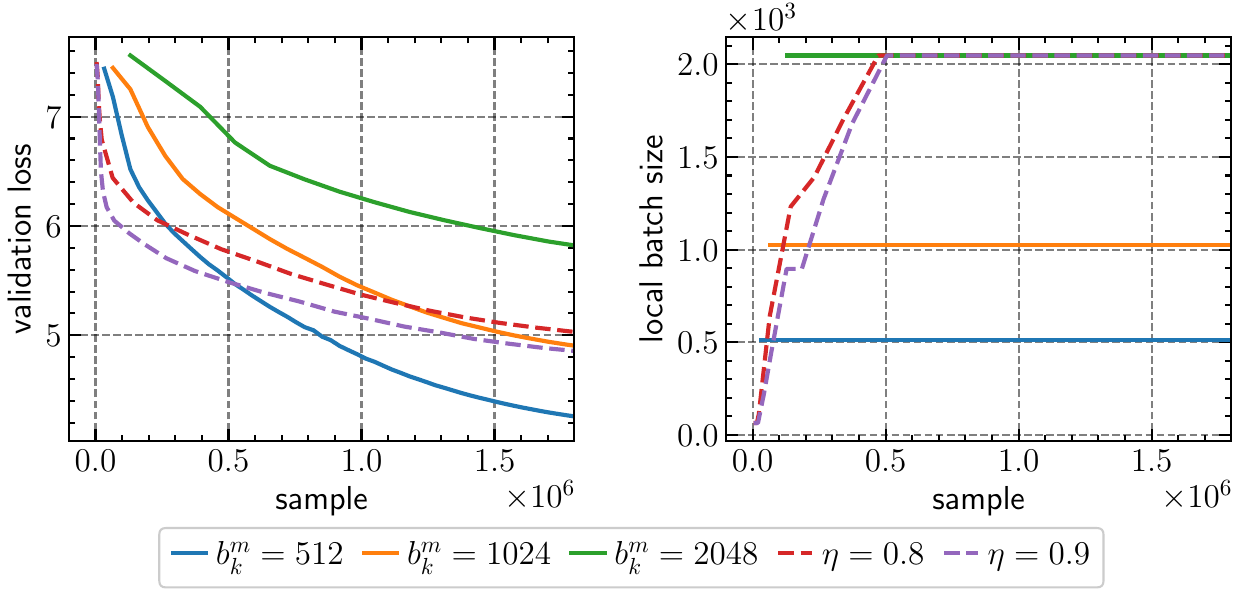}
        \caption{$H=16$ with $b_k^m\in\{512, 1024, 2048\}$ and $\eta\in\{0.8,0.9\}$}
        \label{fig:microllama_const_H16}
    \end{subfigure}%
    \hfill
    \begin{subfigure}[h]{0.49\textwidth}
        \centering
        \includegraphics[width=\textwidth]{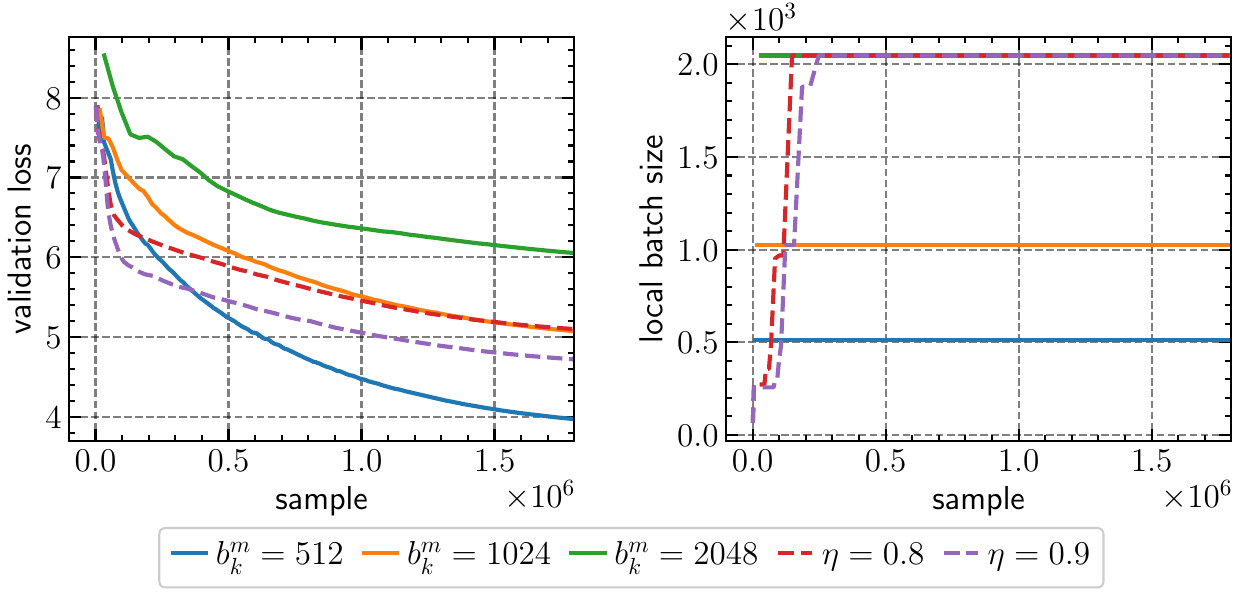}
        \caption{$H=4$ with $b_k^m\in\{512, 1024, 2048\}$ and $\eta\in\{0.8,0.9\}$}
        \label{fig:microllama_const_H4}
    \end{subfigure}%
     \hfill
    \begin{subfigure}[h]{0.49\textwidth}
        \centering
        \includegraphics[width=\textwidth]{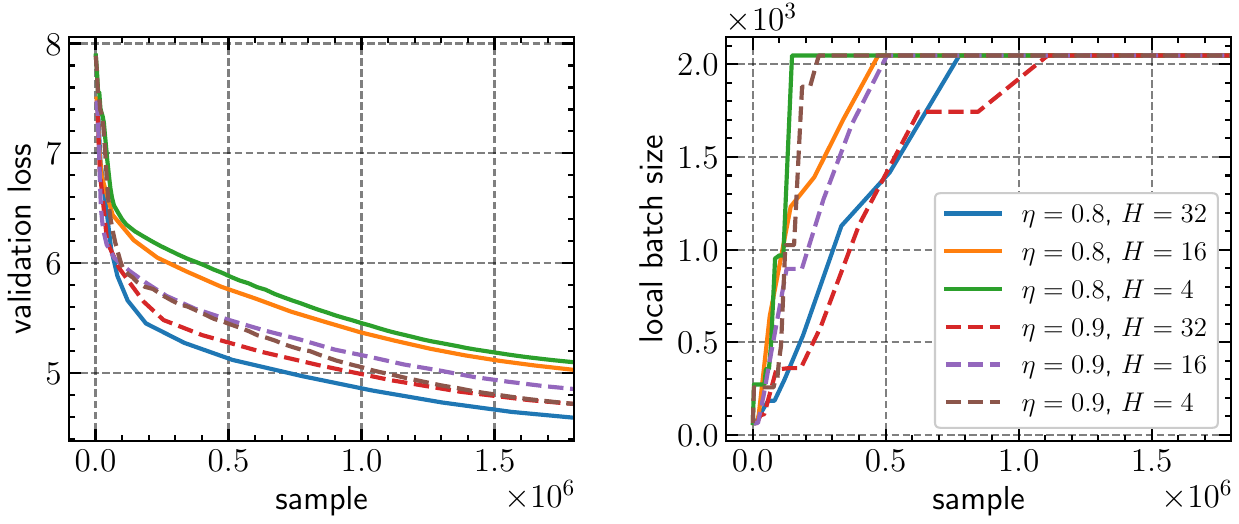}
        \caption{Adaptive batch size strategies for $H\in\{32, 16, 4\}$}
        \label{fig:microllama_final}
    \end{subfigure}
    \caption{Validation loss and local batch sizes of Local \AdamW with adaptive batch size strategies for MicroLlama 300M on C4. }
    \label{fig:microllama}
    \vspace*{-5mm}
\end{figure*}

\section{Concluding Remarks}
\label{sec:conclusion}
This work presents adaptive batch size strategies for local gradient methods under data parallelism to reduce communication overhead by increasing local batch sizes. To address possible efficiency issues in practical implementation, we propose an approximation of the norm test under the context of local gradient methods, by taking advantage of the availability of local batch gradients. 
Illustrated with experiments in both image classification and language modeling task, our proposed approach is able to reduce the generalization gap of large-batch training while only adding a small amount of additional training time. 
Future research will examine convergence in varied settings and federated learning \citep{mcmahan2017communication, kairouz2021advances}, including personalized approaches \citep{hanzely2023personalized}. Asynchronous \SGD \citep{mishchenko2022asynchronous, koloskova2022sharper, nguyen2022federated, leconte2024queuing, islamov2024asgrad} and hierarchical \SGD \citep{lin2020dont, castiglia2021multilevel, wang2022demystifying} extensions are promising for enhancing communication efficiency and tackling worker heterogeneity.

%    \newpage
    \section*{Acknowledgments}
    The research of Tim Tsz-Kit Lau is supported in part through the computational resources and staff contributions from the Data Science Institute at the University of Chicago, through its AI + Science Research Funding Initiatives. 
    The research of Han Liu is supported by NIH R01LM01372201, NSF RI 1840857, NSF TRIPOD 1740735, NSF DMS1454377-CAREER, NSF IIS 1546482, along with an Alfred P.~Sloan Fellowship. 
    The research of Mladen Kolar is supported in part by NSF ECCS-2216912.

%    \newpage    
    \addcontentsline{toc}{section}{\protect\textbf{References}}
    \bibliographystyle{plainnat}
    \bibliography{ref}

    \newpage
    \appendix
    \addcontentsline{toc}{section}{\protect\textbf{Appendix}}
    \begin{center}
        {\LARGE \textsc{Appendix}}
    \end{center}	
    \numberwithin{equation}{section}
    \numberwithin{theorem}{section}
    \numberwithin{proposition}{section}
    \numberwithin{lemma}{section}
    \numberwithin{definition}{section}
    \numberwithin{corollary}{section}
    \numberwithin{example}{section}
    \numberwithin{remark}{section}
    \numberwithin{problem}{section}
    \numberwithin{algorithm}{section}

    \tableofcontents
    
    \newpage
    \section{Main Algorithm}
    \label{sec:algorithm}
    
    \begin{algorithm}[h!]
        \caption{Adaptive Local Batch Size Schedules for Local Gradient Methods}
        \label{alg:full}
        \begin{algorithmic}
            \Require no.~of workers $M\in\NN^*$, no.~of communication steps $K\in\NN^*$, no.~of local steps $H\in\NN^*$, no.~of training samples $N\in\NN^*$, global step counter $k=0$, processed samples counter $B=0$, training sets $(\calD_m)_{m\in\setM}=(\{\xi_i^m\}_{i=1}^{n_m})_{m\in\setM}$, initial model parameters $x_{0,0}^m=x_0\in\RR^d$ for all $m\in\setM$, initial local batch sizes $(b_{0,0}^m)_{m\in\setM}$, learning rate schedules $(\alpha_{k,h})_{k\in\set{0,K-1}, h\in\set{0,H-1}}$, $(\eta_m)_{m\in\setM}\subset(0,1)$
            \While{$B < N$}            
            \ForAll{worker $m=1,\ldots,M$ in parallel}
            \For{local step $h=0, \ldots, H-1$}
            \State Sample the local batch $\calB_{k,h}^m$ of size $b_{k,h}^m\coloneqq|\calB_{k,h}^m|$ uniformly from $\calD_m$
            \State Compute per-sample gradients $\nabla f_m(x_{k,h}^m; \xi_i^m)$ for each $i\in\calB_{k,h}^m$
            \State Compute the local batch gradient $\nabla F_{\calB_{k,h}^m}(x_{k,h}^m) = \frac{1}{b_{k,h}^m} \sum_{i\in\calB_{k,h}^m} \nabla f_m(x_{k,h}^m; \xi_i^m)$
            \State Compute $\Var_{i\in\calB_{k,h}^m}(\nabla f_m(x_{k,h}^m; \xi_i^m))$ by applying \eqref{eqn:def_var}                        
            \State Compute the local norm test statistic $\sfT_{k,h}^m \coloneqq \left \lceil \left(\Var_{i\in\calB_{k,h}^m}(\nabla f_m(x_{k,h}^m; \xi_i^m))\right)\Big/\left(\eta^2  \|\nabla F_{\calB_{k,h}^m}(x_{k,h}^m)\|^2\right) \right\rceil$
            \State Determine the next local batch sizes: $b_{k,h+1}^m = \max\{\sfT_{k,h}^m, b_{k,h}^m\}$ 
            \State Local \SGD step: $x_{k,h+1}^m = x_{k,h}^m - \alpha_{k,h} \nabla F_{\calB_{k,h}^m}(x_{k,h}^m)$ \Comment{or other inner optimizers}
            \EndFor     
            \EndFor
            \State Update the global model: $x_{k+1, 0}^m = \frac1M\sumM x_{k,H}^m$ for each $m\in\setM$ \Comment{all-reduce}
            \State $B \leftarrow B + \sumM\sum_{h=0}^{H-1} b_{k,h}^m$
            \State $k \leftarrow k + 1$                
            \EndWhile
            \Ensure $\hat{x} = x_{K, 0}^1$, where $K$ is the smallest integer such that $\sum_{k=0}^{K-1}\sumM\sum_{h=0}^{H-1} b_{k,h}^m\ge N$
        \end{algorithmic}        
    \end{algorithm}

    \begin{algorithm}[h!]
        \caption{Adaptive Batch Size Schedules for Local  Gradient Methods (Actual Implementation)}
        \label{alg:implemented}
        \begin{algorithmic}
            \Require no.~of workers $M\in\NN^*$, no.~of communication steps $K\in\NN^*$, no.~of local steps $H\in\NN^*$, no.~of training samples $N\in\NN^*$, global step counter $k=0$, processed samples counter $B=0$, training sets $(\calD_m)_{m\in\setM}=(\{\xi_i^m\}_{i=1}^{n_m})_{m\in\setM}$, initial model parameters $x_{0,0}^m=x_0\in\RR^d$ for all $m\in\setM$, initial local batch sizes $(b_0^m)_{m\in\setM}$, learning rate schedules $(\alpha_{k,h})_{k\in\set{0,K-1}, h\in\set{0,H-1}}$, $(\eta_m)_{m\in\setM}\subset(0,1)$
            \While{$B < N$}            
            \ForAll{worker $m=1,\ldots,M$ in parallel}
            \For{local step $h=0, \ldots, H-1$}
            \State Sample the local batch $\calB_{k,h}^m$ of size $b_k^m\coloneqq|\calB_{k,h}^m|$ uniformly from $\calD_m$
            \State Compute the local batch gradient $\nabla F_{\calB_{k,h}^m}(x_{k,h}^m) = \frac{1}{b_{k,h}^m} \sum_{i\in\calB_{k,h}^m} \nabla f_m(x_{k,h}^m; \xi_i^m)$
            \State Local \SGD step: $x_{k,h+1}^m = x_{k,h}^m - \alpha_{k,h} \nabla F_{\calB_{k,h}^m}(x_{k,h}^m)$ \Comment{or other inner optimizers}
            \EndFor     
            \EndFor
            \State Update the global model: $x_{k+1, 0}^m = \frac1M\sumM x_{k,H}^m$ for each $m\in\setM$ \Comment{all-reduce}
            \State Compute the averaged batch gradient of all workers: $\gbar_{k,H} \coloneqq \frac1M\sumM \nabla F_{\calB_{k,H}^m}(x_{k,H}^m)$ \Comment{all-reduce}
            \State Compute $\Var_{i\in\calB_k}\left(\nabla f(x_k^m; \xi_i^m)\right) 
            = b_k^m\cdot \frac1{M-1}\sumM \norm{\nabla F_{\calB_{k,H}^m}(x_{k,H}^m) - \gbar_{k,H}}^2$ 
            \State Compute the norm test statistic $\sfT_k^m\coloneqq\left \lceil \left(\Var_{i\in\calB_k}\left(\nabla f(x_{k,H}^m; \xi_i^m) \right) \right) \Big/ \left(M\eta^2 \norm{\gbar_{k,H}}^2 \right) \right\rceil$ 
            \State Determine the next local batch sizes: $b_{k+1}^m = \max\{\sfT_k^m, b_k^m\}$ 
            \State $B \leftarrow B + HM b_k^m$
            \State $k \leftarrow k + 1$                
            \EndWhile
            \Ensure $\hat{x} = x_{K, 0}^1$, where $K$ is the smallest integer such that $HM\sum_{k=0}^{K-1} b_k^m\ge N$
        \end{algorithmic}        
    \end{algorithm}

    \newpage
    \section{Proofs of Main Text}
    \label{sec:proofs}
    We present the omitted proofs of the main text, mainly by simplifying the technique in \citet{stich2020error}. 
    
    \paragraph{Notation.}
    Let us define a sequence of (virtual) averaged iterates $(\xbar_k)_{k\in\NN}$, given by 
    \begin{equation}\label{eqn:xbar}
        (\forall k\in\NN)\quad\xbar_{k+1} \coloneqq \xbar_k - \alpha_k\gbark, 
    \end{equation}
    where $\gbark \coloneqq\frac1M\sumM g_k^m$ with $g_k^m\coloneqq\nabla F_{\calB_k^m}(x_k^m)$ and $\xbar_0 \coloneqq x_0$. 
    
    Let use recall that $\Ex_k\coloneqq\Ex[\cdot\mid\calF_k]$, $\Ex_{\calF_k^m}\coloneqq\Ex[\cdot\mid\calF_k^m]$, and $\Ex$ denotes the unconditional expectation. 
    
    We also define the notion of slowly decreasing learning rates, called \emph{$\tau$-slow sequences}. 
    
    \begin{definition}[$\tau$-slow sequence; Definition 10 in \citealp{stich2020error}]
        A sequence $(a_k)_{k\in\NN}\subset(0,\infty)$ with positive values is $\tau$-slow decreasing for $\tau\ge1$ if 
        \[(\forall k\in\NN)\quad a_{k+1}\le a_k \qquad \text{and}\qquad a_{k+1}\left(1+\frac{1}{2\tau}\right)\ge a_k. \]
        The sequence $(a_k)_{k\in\NN}$ is $\tau$-slow increasing if $(1/a_k)_{k\in\NN}$ is $\tau$-slow decreasing. 
    \end{definition}
    
    \subsection{Technical Lemmas}
    We first state without proofs some preparatory lemmas, which are either standard results or whose proofs are be found in the cited work.

    \begin{lemma}
        For any $M$ vectors $v_1, \ldots, v_M\in\RR^d$, we have 
        \begin{equation}\label{eqn:sum_norm}
            \norm{\sumM v_m} \le \sumM \norm{v_m}. 
        \end{equation}       
    \end{lemma}
    
    \begin{lemma}
        For any $M$ vectors $v_1, \ldots, v_M\in\RR^d$, we have 
        \begin{equation}\label{eqn:sum_norm_sq}
            \norm{\sumM v_m}^2 \le M\sumM \norm{v_m}^2. 
        \end{equation}       
    \end{lemma}
    
    \begin{lemma}
        For any two vectors $x,y\in\RR^d$ and any positive number $\rho>0$, we have 
        \begin{equation}\label{eqn:jensen}
            \norm{x+y}^2 \le (1+\rho)\norm{a}^2 + (1+1/\rho)\norm{b}^2. 
        \end{equation}       
    \end{lemma}
    
    \begin{lemma}[Cauchy--Schwarz inequality]
        For any two vectors $x,y\in\RR^d$ and any positive number $\rho>0$, we have 
        \begin{equation}\label{eqn:cs}
            2\dotp{x}{y} \le \rho\norm{a}^2 + \rho^{-1}\norm{b}^2. 
        \end{equation}
    \end{lemma}
    
    \begin{lemma}
        If $F\colon\RR^d\to\RR$ is $L$-Lipschitz smooth for some $L>0$ (\Cref{ass:smooth}), then we have 
        \begin{equation}\label{eqn:smooth_alt}
            (\forall x\in\RR^d)\quad\norm{\nabla F(x)}^2 \le 2L(F(x) -\Fstar). 
        \end{equation}
    \end{lemma}
    
    \begin{lemma}[Lemma 2 in \citealp{stich2019unified}; Lemma 13 in \citealp{stich2020error}]\label{lemma:linear}
        For any nonnegative sequence $(u_k)_{k\in\NN}$ and parameters $\kappa\ge\chi>0$, $\sigma\ge0$ and $K\in\NN$, there exists a constant $\alpha\le1/\kappa$ such that for constant learning rates $(\alpha_k\equiv\alpha)_{k\in\NN}$ and weights $w_k\coloneqq(1-\chi\alpha)^{-k-1}$, we have 
        \[\Psi_K \coloneqq \frac{1}{W_K}\sum_{k=0}^K w_k\left(\frac{(1-\chi\alpha_k)u_k}{\alpha_k} - \frac{u_{k+1}}{\alpha_k} + \sigma\alpha_k\right) = \tilde{\scrO}\left(\kappa u_0 \e^{-\chi K/\kappa} + \frac{\sigma}{\chi K}\right), \]
        where $W_K\coloneqq\sumK w_k$. 
        
    \end{lemma}

    \begin{lemma}[Lemma 14 in \citealp{stich2020error}]\label{lemma:sublinear}
        For any nonnegative sequence $(u_k)_{k\in\NN}$ and parameters $\kappa\ge0$, $\sigma\ge0$ and $K\in\NN$, there exists a constant $\alpha\le1/\kappa$ such that for constant learning rates $(\alpha_k\equiv\alpha)_{k\in\NN}$, we have 
        \[\Psi_K\coloneqq \frac{1}{K+1}\sum_{k=0}^K \left(\frac{u_k}{\alpha_k} - \frac{u_{k+1}}{\alpha_k} + \sigma\alpha_k\right) \le \frac{\kappa u_0}{K+1} + 2\sqrt{\frac{\sigma u_0}{K+1}}. \]
        
    \end{lemma}

    \subsection{Proofs of \Cref{thm:conv_strong_cvx,thm:conv_cvx,thm:conv_noncvx}}
    We first prove a descent lemma for ($\mu$-strongly) convex objectives. 
    
    \begin{lemma}[Descent lemma for ($\mu$-strongly) convex objectives]\label{lemma:descent_strong_cvx}   
        Suppose that \Cref{ass:smooth,ass:strong_cvx} hold with $\mu>0$. 
        Let $(x_k^m)_{k\in\NN, m\in\setM}$ be the sequence of the iterates of Local \SGD \eqref{eqn:local_sgd} with the (exact variance) local norm test \eqref{eqn:exact_norm_local} with $\eta_m\equiv\eta\in(0,1)$. 
        If the sequence of learning rates $(\alpha_k)_{k\in\NN}$ satisfies $\alpha_k\le M/(4L(M+\eta^2))$ for all $k\in\NN$, then the sequence $(\xbar_k)_{k\in\NN}$ defined in \eqref{eqn:xbar} satisfies
        \begin{equation*}
            \Ex\norm{\xbar_{k+1} - x^\star}^2 
            \le \left(1 - \frac{\mu\alpha_k}{2}\right)\Ex\norm{\xbar_k - x^\star}^2 - \frac{\alpha_k}{2M}\sumM \Ex[F(x_k^m) - F^\star] + \frac{3L\alpha_k}{M} \sumM \Ex\norm{x_k^m - \xbar_k}^2. 
        \end{equation*}
        
    \end{lemma}
    
    \begin{proof}
        The proof is similar to that of Lemma 25 in \citet{stich2020error}. We state it here for convenience and completeness. 
        Expanding we have
        \begin{equation*}
            \norm{\xbar_{k+1} - x^\star}^2 
            = \norm{\xbar_{k} - x^\star}^2 - \frac{2\alpha_k}{M}\sumM \dotp{g_k^m}{x_k^m - \xstar} + \frac{\alpha_k^2}{M^2}\norm{\sumM \gkm}^2 + \frac{2\alpha_k}{M}\sumM\dotp{\gkm}{\xkm - \xbark}. 
        \end{equation*}
        Taking expectation with respect to $\calF_k$, we have 
        \begin{multline*}
            \Ex_k \norm{\xbar_{k+1} - x^\star}^2 \\
            \le \norm{\xbar_{k} - x^\star}^2 - \frac{2\alpha_k}{M}\sumM \dotp{\gradFxkm}{x_k^m - \xstar} + \frac{\alpha_k^2}{M^2}\Ex_k\norm{\sumM \gkm}^2+ \frac{2\alpha_k}{M}\sumM\dotp{\gradFxkm}{\xkm - \xbark}. 
        \end{multline*}
        Note that $\Ex_k\norm{\sumM \gkm}^2$ can be bounded by 
        \begin{align}
            \Ex_k\norm{\sumM \gkm}^2 & = \norm{\sumM \gradFxkm}^2 + \sumM \Ex\norm{\gkm - \gradFxkm}^2 \nonumber \\
            &\le M\sumM\norm{\gradFxkm}^2 + \eta^2\sumM \norm{\gradFxkm}^2 & \text{by \eqref{eqn:sum_norm_sq} and \eqref{eqn:exact_norm_actual}} \label{eqn:bound_sum_grad_norm_sq_1} \\
            &\le 2L(M + \eta^2)\sumM \left( F(\xkm) - \Fstar\right)  & \text{by \eqref{eqn:smooth_alt}} \label{eqn:bound_sum_grad_norm_sq_2}
        \end{align}
        On the other hand, by \Cref{ass:strong_cvx}, we also have 
        \begin{equation}\label{eqn:strong_cvx_1}
            -2\dotp{\gradFxkm}{\xkm - \xstar} \le - \mu\norm{\xkm - \xstar}^2 -2 \left( F(\xkm) - \Fstar\right). 
        \end{equation}
        By Cauchy--Schwarz inequality \eqref{eqn:cs} with $\rho=2L$, we have 
        \begin{align}
            2\dotp{\gradFxkm}{\xkm-\xbark} &\le 2L\norm{\xkm-\xstar}^2 + \frac{1}{2L}\norm{\gradFxkm}^2 \nonumber \\
            &\le 2L\norm{\xkm-\xstar}^2 + F(\xkm) - \Fstar & \text{by \eqref{eqn:smooth_alt}} \label{eqn:cs_2}
        \end{align}
        By \eqref{eqn:jensen}, we also have 
        \begin{equation}\label{eqn:jensen_2}
            -\norm{\xkm-\xstar}^2 \le -\frac12\norm{\xbark - \xstar}^2 + \norm{\xkm - \xbark}^2. 
        \end{equation}
        
        Thus, plugging in \eqref{eqn:bound_sum_grad_norm_sq_2}, \eqref{eqn:cs_2} and \eqref{eqn:jensen_2}, we obtain 
        \begin{align*}
            \Ex_k \norm{\xbar_{k+1} - x^\star}^2 
            &\le \norm{\xbar_{k} - x^\star}^2 - \frac{2\alpha_k}{M}\sumM \dotp{\gradFxkm}{x_k^m - \xstar} + \frac{2L\alpha_k^2(M + \eta^2)}{M^2}\sumM \left( F(\xkm) - \Fstar\right) \\
            &\qquad+ \frac{2\alpha_k}{M}\sumM\dotp{\gradFxkm}{\xkm - \xbark} \\
            &\le \norm{\xbar_{k} - x^\star}^2 - \frac{\alpha_k}{M}\sumM \left[ \mu\norm{\xkm - \xstar}^2 +2 \left( F(\xkm) - \Fstar\right)\right] \\
            &\qquad+ \frac{2L(M + \eta^2)\alpha_k^2}{M^2}\sumM \left( F(\xkm) - \Fstar\right) + \frac{\alpha_k}{M}\sumM\left[2L\norm{\xkm-\xstar}^2 + F(\xkm) - \Fstar \right] \\
            &\le \norm{\xbar_{k} - x^\star}^2 + \frac{\mu\alpha_k}{M}\sumM\left[-\frac12\norm{\xbark - \xstar}^2 + \norm{\xkm - \xbark}^2\right] \\
            &\qquad-\frac{\alpha_k}{M}\left(1-\frac{2L(M+\eta^2)\alpha_k}{M}\right)\sumM \left( F(\xkm) - \Fstar\right) + \frac{2L\alpha_k}{M}\sumM\norm{\xkm - \xbark}^2 \\
            &\le \left(1-\frac{\mu\alpha_k}{2}\right)\norm{\xbar_{k} - x^\star}^2 - \frac{\alpha_k}{2M}\sumM \left( F(\xkm) - \Fstar\right)+ \frac{3L\alpha_k}{M} \sumM\norm{x_k^m - \xbar_k}^2, 
        \end{align*}
        since $\alpha_k\le M/(4L(M+\eta^2))$ and $\mu\le L$ if $F$ is both $\mu$-strongly convex and $L$-Lipschitz smooth. Taking total expectation gives the desired result.         
    \end{proof}

    We also prove a descent lemma for general nonconvex objectives.     
    \begin{lemma}[Descent lemma for nonconvex objectives]\label{lemma:descent_noncvx}   
        Suppose that \Cref{ass:smooth} holds. 
        Let $(x_k^m)_{k\in\NN, m\in\setM}$ be the sequence of the iterates of Local \SGD \eqref{eqn:local_sgd} with the (exact variance) local norm test \eqref{eqn:exact_norm_local} with $\eta_m\equiv\eta\in(0,1)$. 
        If the sequence of learning rates $(\alpha_k)_{k\in\NN}$ satisfies $\alpha_k\le M/(2L(M+\eta^2))$ for all $k\in\NN$, then the sequence $(\xbar_k)_{k\in\NN}$ defined in \eqref{eqn:xbar} satisfies
        \begin{equation*}
            \Ex[F(\xbar_{k+1})] \le \Ex[F(\xbar_k)] - \frac{\alpha_k}{4M}\sumM\Ex\norm{\nabla F(x_k^m)}^2 + \frac{L^2\alpha_k}{2M}\sumM \Ex\norm{x_k^m - \xbar_k}^2. 
        \end{equation*}
        
    \end{lemma}
    
    \begin{proof}
        The proof is similar to that of Lemma 26 in \citet{stich2020error}. We state it here for convenience and completeness. 
        
        By \Cref{ass:smooth}, we have 
        \[F(\xbarkpo) \le F(\xbark) - \frac{\alpha_k}{M}\sumM\dotp{\gradFxbark}{\gkm} + \frac{L\alpha_k^2}{2M^2}\norm{\sumM \gkm}^2. \]
        Taking expectation with respect to $\calF_k$, we have 
        \begin{align*}
            \Ex_k [F(\xbarkpo)] &\le F(\xbark) - \frac{\alpha_k}{M}\sumM\dotp{\gradFxbark}{\gradFxkm} + \frac{L\alpha_k^2}{2M^2}\Ex_k\norm{\sumM \gkm}^2 \\
            &\le F(\xbark) - \frac{\alpha_k}{M}\sumM\dotp{\gradFxbark}{\gradFxkm} + \frac{L\alpha_k^2(M+\eta^2)}{2M^2}\sumM\norm{\gradFxkm}^2 & \text{by \eqref{eqn:bound_sum_grad_norm_sq_1}} \\
            &= F(\xbark) -\left(\frac{\alpha_k}{M} - \frac{L\alpha_k^2(M+\eta^2)}{2M^2}\right)\sumM\norm{\gradFxkm}^2 + \frac{\alpha_k}{M}\sumM\dotp{\gradFxkm - \gradFxbark}{\gradFxkm}. 
        \end{align*}
        By Cauchy--Schwarz inequality \eqref{eqn:cs} with $\rho=1$ and \Cref{ass:smooth}, we have 
        \begin{align*}
            \sumM\dotp{\gradFxkm - \gradFxbark}{\gradFxkm} &\le \frac12\sumM \left[\norm{\gradFxkm - \gradFxbark}{\gradFxkm}^2 + \norm{\gradFxkm}^2\right] \\
            &\le \frac{L^2}{2}\sumM\norm{\xkm - \xbark}^2 + \frac12\sumM\norm{\gradFxkm}^2. 
        \end{align*}
        Plugging this in, we obtain 
        \begin{align*}
            \Ex_k [F(\xbarkpo)] &= F(\xbark) -\frac{\alpha_k}{2M} \left(1 - \frac{L(M+\eta^2)\alpha_k}{M}\right)\sumM\norm{\gradFxkm}^2 + \frac{L^2\alpha_k}{2M}\sumM\norm{\xkm - \xbark}^2 \\
            &\le F(\xbark) - \frac{\alpha_k}{4M}\sumM\norm{\nabla F(x_k^m)}^2 + \frac{L^2\alpha_k}{2M}\sumM \norm{x_k^m - \xbar_k}^2,  
        \end{align*}
        since $\alpha_k\le M/(2L(M+\eta^2))$. Taking total expectation gives the desired result.         
    \end{proof}

    We then derive an upper bound for $\sumM\Ex\norm{x_k^m - \xbar_k}^2$. 
    \begin{lemma}\label{lemma:bound_norm_sq}   
        Let $(x_k^m)_{k\in\NN, m\in\setM}$ be the sequence of the iterates of Local \SGD \eqref{eqn:local_sgd} with the (exact variance) local norm test \eqref{eqn:exact_norm_local} with $\eta_m\equiv\eta\in(0,1)$, and $(\xbar_k)_{k\in\NN}$ be the sequence defined in \eqref{eqn:xbar}. If the sequence of learning rates $(\alpha_k)_{k\in\NN}$ satisfies $\alpha_k\le 1/(10L(HM+\eta^2))$ for all $k\in\NN$ and $(\alpha_k^2)_{k\in\NN}$ is $H$-slow decreasing, then we have 
        \begin{equation}
            \frac{3L}M\sumM\Ex\norm{x_k^m - \xbar_k}^2 \le \frac{1}{10LHM}\sumM \sum_{h=0}^{H-1} \Ex\norm{\nabla F(x_{k-h}^m)}^2. 
        \end{equation}
        Moreover, for any $H$-slow increasing nonnegative sequence $(w_k)_{k\in\NN}$ we have 
        \[
        \frac{3L}M\sumM\sum_{k=0}^K w_k \Ex\norm{x_k^m - \xbar_k}^2 \le \frac{1}{5LM}\sumM\sum_{k=0}^Kw_k \Ex\norm{\nabla F(x_{k-h}^m)}^2.
        \]
    \end{lemma}    
    
    \begin{proof}
        We first consider
        \begin{align*}
            \frac1M\sumM\Ex\norm{x_k^m - \xbar_k}^2 
            &= \frac1M\sumM\Ex\norm{\xkm - \xbar_{\lfloor K/H\rfloor H} - (\xbark - \xbar_{\lfloor K/H\rfloor H})}^2 \\
            &\le \frac1M\sumM \Ex\norm{\xkm -  \xbar_{\lfloor K/H\rfloor H}}^2 \\
            &\le \frac1M\sumM\Ex\norm{\sum_{h=0}^{H-1}\alpha_{k-h} \left(\nabla F(x_{k-h}^m) + g_{k-h}^m - \nabla F(x_{k-h}^m)\right)}^2 \\
            &\le  \frac{3H}{2M}\sumM\sum_{h=0}^{H-1}\alpha_{k-h}^2\Ex\norm{\nabla F(x_{k-h}^m)}^2 \\
            &\qquad\qquad+ \frac3M\sumM\sum_{h=0}^{H-1}\alpha_{k-h}^2 \Ex\norm{g_{k-h}^m - \nabla F(x_{k-h}^m)}^2 & \text{by \eqref{eqn:jensen} and \eqref{eqn:sum_norm_sq}} \\
            &\le \frac{3(H+\eta^2)}{2K}\sumM\sum_{h=0}^{H-1}\alpha_{k-h}^2\Ex\norm{\nabla F(x_{k-h}^m)}^2. 
        \end{align*}
        The first inequality follows from $\Ex\norm{X-\Ex X}^2 \le \Ex\norm{X}^2$ for any random variable $X$ in $\RR^d$. Since $(\alpha_k^2)_{k\in\NN}$ is $H$-slow decreasing, for $h\le H$, we have $\alpha_{k-h}^2\le \alpha_k^2(1+1/(2H))^H\le\alpha_k^2\e^{H/(2H)} \le 2\alpha_k^2$ since $1+x\le\e^x$ for any $x\in\RR$. Then, we can simply the inequality:
        \begin{equation}
            \frac1M\sumM\Ex\norm{x_k^m - \xbar_k}^2  \le \frac{3(H+\eta^2)\alpha_k^2}{K}\sumM\sum_{h=0}^{H-1}\Ex\norm{\nabla F(x_{k-h}^m)}^2. 
        \end{equation}
        Since $\alpha_k\le 1/(10L(HM+\eta^2))$ for all $k\in\NN$, we have 
        \begin{equation*}
            3L\cdot3(H+\eta^2)\alpha_k^2 \le \frac{1}{10L(HM+\eta^2)} \le \frac{1}{10LHM}. 
        \end{equation*}
        Since $(w_k)_{k\in\NN}$ is $H$-slow increasing, we have $w_k\le w_{k-h}(1+2/(2H))^h\le w_{k-h}(1+2/(2H))^H \le w_{k-h}\sqrt{\e} \le 2w_{k-h}$ for each $h\in\set{0,H}$. Therefore, we have 
        \begin{align}
            \frac{3L}M\sumM\sum_{k=0}^K w_k\Ex\norm{x_k^m - \xbar_k}^2 &\le \frac{1}{5LM}\sumM\sum_{k=0}^K \frac{w_k}{2H}\sum_{h=0}^{H-1} \Ex\norm{\nabla F(x_{k-h}^m)}^2 \nonumber\\
            &\le \frac{1}{5LM}\sumM\sum_{k=0}^K \frac{1}{H}\sum_{h=0}^{H-1} w_{k-h}\Ex\norm{\nabla F(x_{k-h}^m)}^2 \nonumber\\
            &\le \frac{1}{5LM}\sumM\sum_{k=0}^Kw_k \Ex\norm{\nabla F(x_{k-h}^m)}^2. \label{eqn:bound_norm_sq_3}
        \end{align}
    \end{proof}

    \begin{proof}[Proof of \Cref{thm:conv_strong_cvx}]
        Using \Cref{lemma:bound_norm_sq} with $\mu>0$ and \eqref{eqn:smooth_alt}, we have 
        \begin{equation}\label{eqn:bound_norm_sq_2}
            \frac{3L}M\sumM\sum_{k=0}^K w_k \Ex\norm{x_k^m - \xbar_k}^2 \le \frac{2}{5M}\sumM\sum_{k=0}^Kw_k \Ex[F(x_{k-h}^m) - \Fstar].
        \end{equation}
        
        Since $\alpha_k\equiv\alpha\le 1/(10L(HM+\eta^2)) \le M/(4L(M+\eta^2))$, by \Cref{lemma:descent_strong_cvx}, we have 
        \begin{equation*}
            \frac{1}{2M}\sumM \Ex[F(x_k^m) - F^\star] 
            \le \frac1{\alpha_k}\left(1 - \frac{\mu\alpha_k}{2}\right)\Ex\norm{\xbar_k - x^\star}^2 - \frac1{\alpha_k}\Ex\norm{\xbar_{k+1} - x^\star}^2 + \frac{3L}{M} \sumM \Ex\norm{x_k^m - \xbar_k}^2. 
        \end{equation*}
        Multiplying the weight $w_k$ and summing over $k\in\set{0, K}$, and by \eqref{eqn:bound_norm_sq_2}, we have 
        \begin{align*}
            &\quad\,\frac{1}{2M}\sumM \sum_{k=0}^K w_k\Ex[F(x_k^m) - F^\star] \\
            &\le \sum_{k=0}^K \left( \frac{w_k}{\alpha_k}\left(1 - \frac{\mu\alpha_k}{2}\right)\Ex\norm{\xbar_k - x^\star}^2 - \frac{w_k}{\alpha_k}\Ex\norm{\xbar_{k+1} - x^\star}^2\right) +  \frac{3L}{M} \sumM\sum_{k=0}^Kw_k \Ex\norm{x_k^m - \xbar_k}^2 \\
            &\le \sum_{k=0}^K \left( \frac{w_k}{\alpha_k}\left(1 - \frac{\mu\alpha_k}{2}\right)\Ex\norm{\xbar_k - x^\star}^2 - \frac{w_k}{\alpha_k}\Ex\norm{\xbar_{k+1} - x^\star}^2\right) + \frac{2}{5M}\sumM\sum_{k=0}^Kw_k \Ex[F(x_{k-h}^m) - \Fstar], 
        \end{align*}
        which is equivalent to 
        \begin{equation}\label{eqn:bound_gap}
            \frac{1}{MW_K}\sumM \sum_{k=0}^K w_k\Ex[F(x_k^m) - F^\star] 
            \le \frac{10}{W_K}\sum_{k=0}^K \left( \frac{w_k}{\alpha_k}\left(1 - \frac{\mu\alpha_k}{2}\right)\Ex\norm{\xbar_k - x^\star}^2 - \frac{w_k}{\alpha_k}\Ex\norm{\xbar_{k+1} - x^\star}^2\right).  
        \end{equation}        
        We observe that $1-\mu\alpha/2 \ge 1 - \mu/(20L(HM+\eta^2)) \ge 1-1/(8HM)$, we choose $w_k = (1-\mu\alpha/2)^{-k-1}$ which are $2H$-slow increasing. Then by \Cref{lemma:linear}, we obtain the desired result. 
    \end{proof}

    \begin{proof}[Proof of \Cref{thm:conv_cvx}]
        For the convex case with $\mu=0$, we simply use \Cref{lemma:linear} with $w_k\equiv1$ for all $k\in\NN$. 
    \end{proof}

    \begin{proof}[Proof of \Cref{thm:conv_noncvx}]
        By \Cref{lemma:descent_noncvx} and \eqref{eqn:smooth_alt}, and define $u_k\coloneqq 4\cdot\Ex[F(\xbark) - F^\star]$, we have 
        \begin{align*}
            \frac{1}{MW_K}\sumM\sum_{k=0}^K w_k \Ex\norm{\nabla F(x_k^m)}^2 
            &\le \frac{1}{W_K}\sum_{k=0}^K w_k\left(\frac{u_k}{4\alpha_k} -\frac{u_{k+1}}{4\alpha_k}\right) + \frac{2L\alpha_k}{MW_K}\sumM\sum_{k=0}^K w_k \Ex\norm{x_k^m - \xbar_k}^2 \\
            &\le \frac{1}{W_K}\sum_{k=0}^Kw_k \left(\frac{u_k}{4\alpha_k} -\frac{u_{k+1}}{4\alpha_k}\right) + \frac{2\alpha_k}{15LMW_K}\sumM\sum_{k=0}^Kw_k \Ex\norm{\nabla F(x_k^m)}^2, 
        \end{align*}
        by \eqref{eqn:bound_norm_sq_3}. We simplify this to 
        \[\left(1- \frac{2\alpha_k}{15L}\right)\frac{1}{MW_K}\sumM\sum_{k=0}^K w_k \Ex\norm{\nabla F(x_k^m)}^2 \le \frac{1}{W_K}\sum_{k=0}^Kw_k \left(\frac{u_k}{4\alpha_k} -\frac{u_{k+1}}{4\alpha_k}\right). \]
        Applying \Cref{lemma:sublinear} with $w_k\equiv1$ for all $k\in\NN$ gives the desired result.

    \end{proof}

    \newpage
    \section{Details of Numerical Experiments}
    \label{sec:expt_details}
    For experiments with 4 GPUs, each node has either NVIDIA L40S (48GB; for \ResNet-50 on CIFAR-10) or A40 (48GB; for MicroLlama 300M on C4) or A100 (SXM4 80GB; for \ResNet-101 on ImageNet) GPUs. 

    \subsection{\ResNet-50 on CIFAR-10}
    \label{subsec:resnet-50_cifar10_supp}
    
    \begin{table}[h!]
        \centering
        \caption{Training hyperparameters for \ResNet-50 on CIFAR-10}
        \label{table:hyperparams_resnet-50_cifar-10}
        \vspace*{2mm}
        \begin{tabular}{lc}
            Model & \ResNet-50 on CIFAR-10 \\
            \midrule
            Training samples & 30,000,000 (600 epochs) \\
            Weight initialization & Default  \\
            Optimizer &  Momentum \SGD (\SHB) \\
            Learning rate schedule & Linear warmup + cosine decay \\
            Learning rate warmup (samples) & 1,000,000 (10\%) \\
            Peak learning rate & 0.05 \\           
            Base learning rate & 0.005 \\
            Learning rate scaling rule & linear (for constant batch sizes) \\
            Base global batch size & 256 \\
            Base local batch size & 64 \\
            Maximum global batch size & 50,000 \\
            Maximum local batch size & 12,500 \\
            Weight decay & 0.0001 \\
            Momentum & 0.9 \\
            Precision & \texttt{bfloat16} \\
            Data-parallel size & 4 \\
            \bottomrule
        \end{tabular}
    \end{table}

    \begin{figure}[h!]    
        \centering
        \begin{subfigure}[h]{0.7\textwidth}
            \centering
            \includegraphics[width=\textwidth]{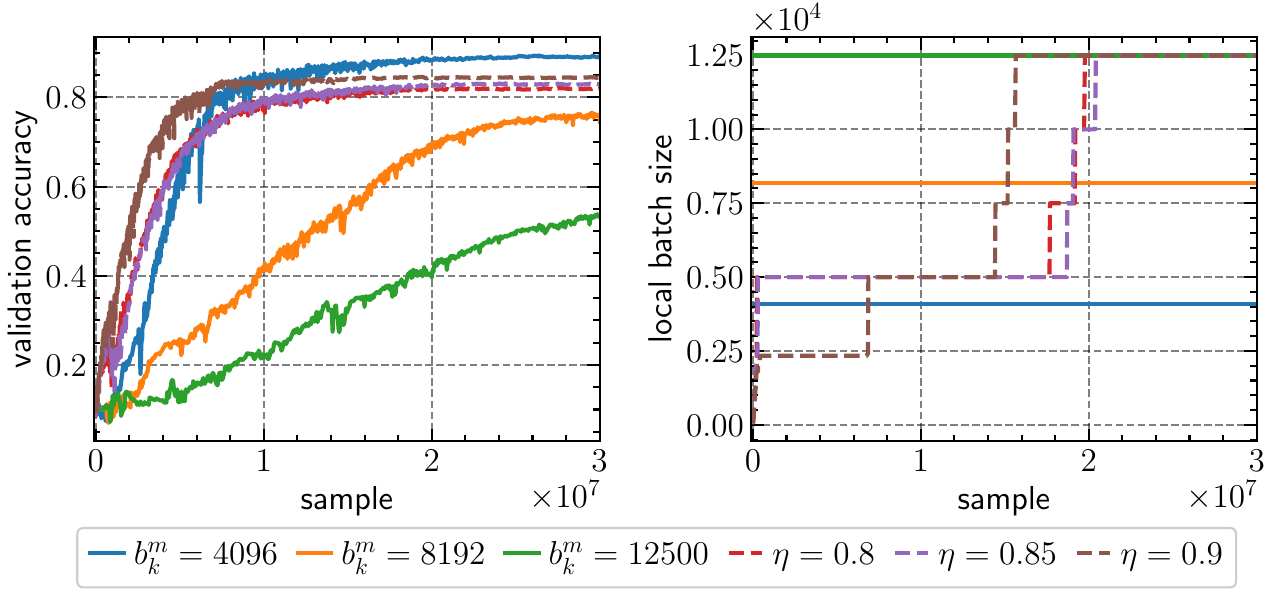}
            \caption{$H=1$ with $b_k^m\in\{4096, 8192, 12500\}$ and $\eta\in\{0.8,0.85,0.9\}$}
            \label{fig:resnet-50_cifar10_const_H1}
        \end{subfigure}%
        \caption{Validation accuracy and local batch sizes of Local \SHB with adaptive batch size strategies for \ResNet-50 on CIFAR-10. }
        \label{fig:resnet-50_cifar10_H1}
    \end{figure}

    \clearpage
    \newpage
    \begin{table}[h!]
        \centering
        \caption{Adaptive batch size strategies for \ResNet-50 on CIFAR-10; mean (standard deviation) of no.~of total steps (steps), wall-clock time (in hours), average local batch sizes (bsz.), and best validation accuracy (acc.; in \%) of all three seeds}
        \label{table:resnet-50_cifar10_full}
        \vspace*{-2mm}
        \scriptsize
        \renewcommand{\arraystretch}{1.}
        \begin{tabular}{l|rrrr|rrrr}
            \toprule[1pt]
            \multirow{2}{*}{Schedule}
            & \multicolumn{4}{c|}{$H=32$}
            & \multicolumn{4}{c}{$H=16$}
            \\
            & steps  & time & bsz. & acc. & steps  & time & bsz. & acc. \\
            \midrule
            Constant & 1824 & 0.98 (0.00) & 4096 & 75.23 (7.19) & 1824 & 0.99 (0.00) & 4096 & 79.62 (4.39) \\ 
            Constant & 896 & 0.95 (0.00) & 8192 & 33.92 (10.78) & 912 & 0.98 (0.00) & 8192 & 31.23 (14.77) \\ 
            Constant & 576 & 1.07 (0.00) & 12500 & 10.70 (0.79) & 592 & 1.10 (0.00) & 12500 & 14.92 (5.49) \\ 
            $\eta=0.8$ & 950 (98) & 1.14 (0.02) & 7787 (607) & 71.49 (3.09) & 827 (25) & 1.14 (0.01) & 9030 (209) & 72.15 (3.82) \\ 
            $\eta=0.85$ & 1088 (0) & 1.16 (0.02) & 6943 (129) & 70.54 (1.03) & 880 (43) & 1.14 (0.01) & 8457 (374) & 72.41 (2.52) \\ 
            $\eta=0.9$ & 1163 (67) & 1.14 (0.01) & 6346 (341) & 73.02 (2.82) & 1072 (43) & 1.15 (0.00) & 7027 (238) & 73.22 (3.72) \\ 
            \midrule
            \multirow{2}{*}{Schedule}
            & \multicolumn{4}{c|}{$H=4$}
            & \multicolumn{4}{c}{$H=1$}
            \\
            & steps  & time & bsz. & acc. & steps  & time & bsz. & acc. \\
            \midrule 
            Constant & 1828 & 1.06 (0.00) & 4096 & 88.72 (0.52) & 1831 & 1.34 (0.00) & 4096 & 89.31 (0.34) \\ 
            Constant & 912 & 1.01 (0.00) & 8192 & 76.56 (2.29) & 915 & 1.15 (0.00) & 8192 & 78.19 (2.63) \\ 
            Constant & 596 & 1.14 (0.00) & 12500 & 35.70 (11.59) & 599 & 1.23 (0.00) & 12500 & 40.60 (15.22) \\ 
            $\eta=0.8$ & 750 (25) & 1.16 (0.00) & 10001 (346) & 74.66 (0.88) & 1023 (211) & 1.36 (0.05) & 7548 (1554) & 80.56 (2.08) \\ 
            $\eta=0.85$ & 771 (41) & 1.16 (0.01) & 9736 (494) & 75.57 (0.21) & 935 (297) & 1.35 (0.08) & 8525 (2399) & 78.89 (4.03) \\  
            $\eta=0.9$ & 756 (14) & 1.16 (0.01) & 9916 (166) & 74.77 (0.42) & 1246 (354) & 1.42 (0.07) & 6415 (2094) & 81.73 (3.65) \\ 
            \bottomrule
        \end{tabular}
    \end{table}

    \paragraph{Observations.}
    Note that due to the stochastic nature of drawing data samples when computing gradients, the adaptive batch size strategies might give rise to different adaptive batch size schedules when different random seeds. We hence perform the training runs for three different random seeds to evaluate the robustness of the proposed strategies (see results of the other two seeds in \Cref{fig:resnet-50_cifar10_2,fig:resnet-50_cifar10_3}). We observe from \Cref{table:resnet-50_cifar10_full} that, as the number of local steps $H$ grows, constant batch sizes might have more varied performance (in terms of standard deviation of best validation accuracy) than the proposed strategies.

    Also notice that, due to the use of relatively large batch sizes, all the experiments make use of the technique of gradient accumulation, which essentially consists of serial operations using for-loops. Therefore, the benefit of using large batch sizes is not reflected in the total training time. If a larger number of GPUs or stronger GPUs with more memory are used, training with large batch sizes will consume substantially shorter time. Consequently, it is better to measure the computational efficiency of different training runs using the number of total training steps  \citep{shallue2019measuring}.

    % \vfill
    \begin{figure}[h!]    
        \centering
        \begin{subfigure}[h]{0.49\textwidth}
            \centering
            \includegraphics[width=\textwidth]{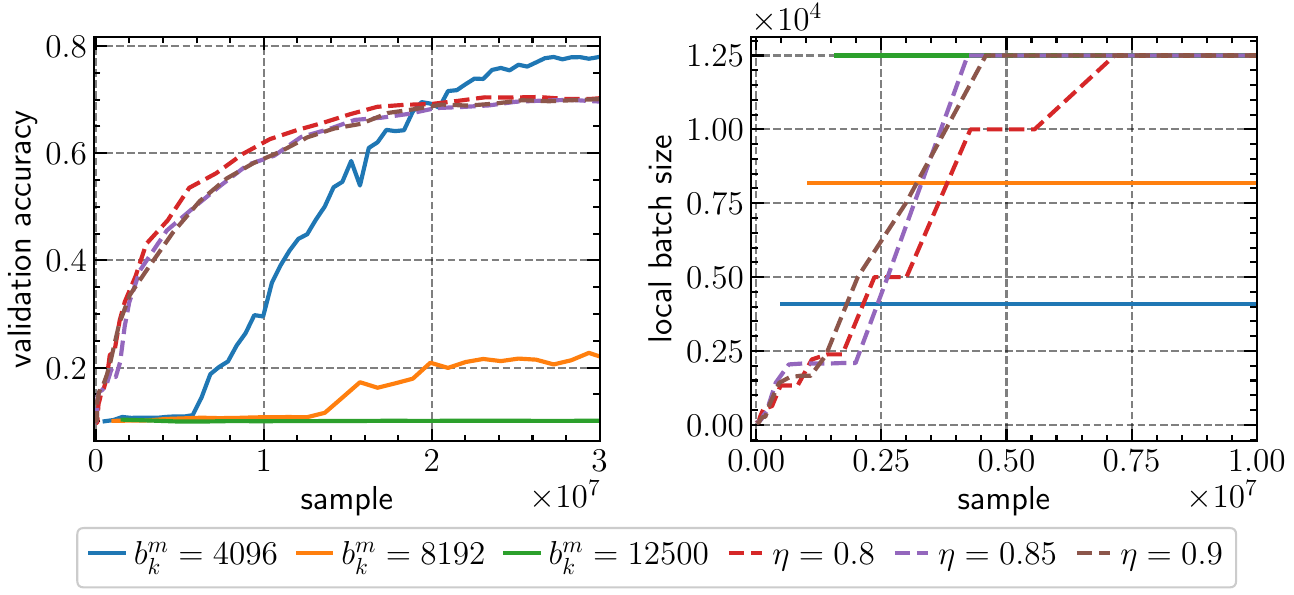}
            \caption{$H=32$ with $b_k^m\in\{4096, 8192, 12500\}$ and $\eta\in\{0.8,0.85,0.9\}$}
            \label{fig:resnet-50_cifar10_const_H32_2}
        \end{subfigure}%
        \hfill
        \begin{subfigure}[h]{0.49\textwidth}
            \centering
            \includegraphics[width=\textwidth]{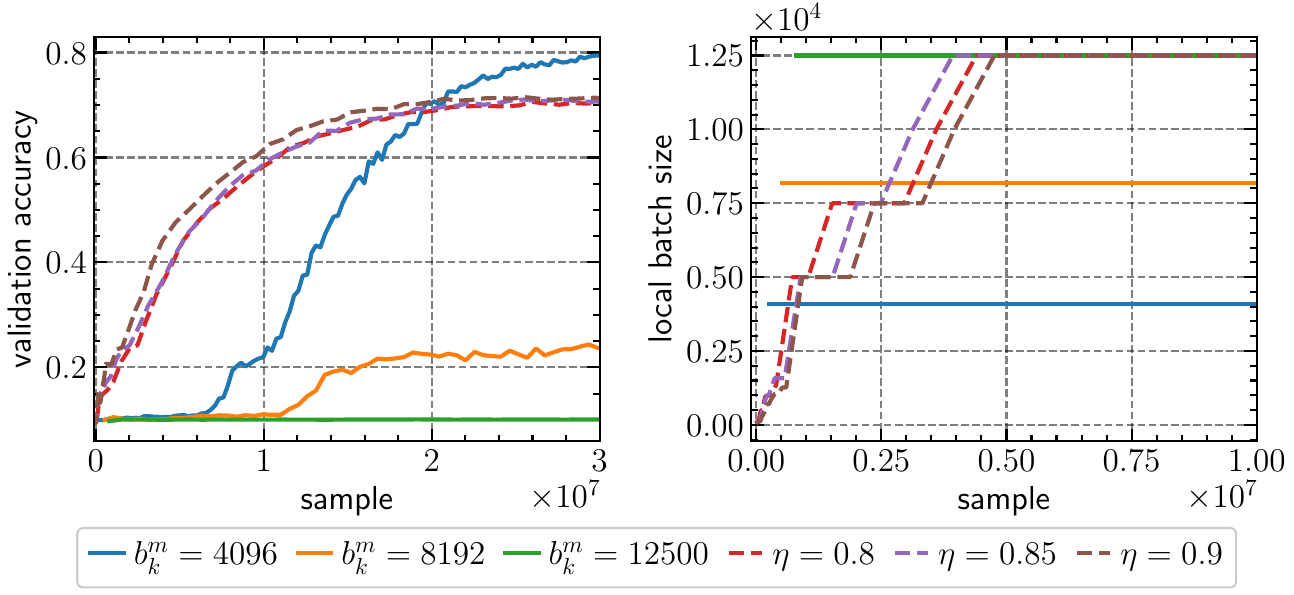}
            \caption{$H=16$ with $b_k^m\in\{4096, 8192, 12500\}$ and $\eta\in\{0.8,0.85,0.9\}$}
            \label{fig:resnet-50_cifar10_const_H16_2}
        \end{subfigure}%
        \hfill
        \begin{subfigure}[h]{0.49\textwidth}
            \centering
            \includegraphics[width=\textwidth]{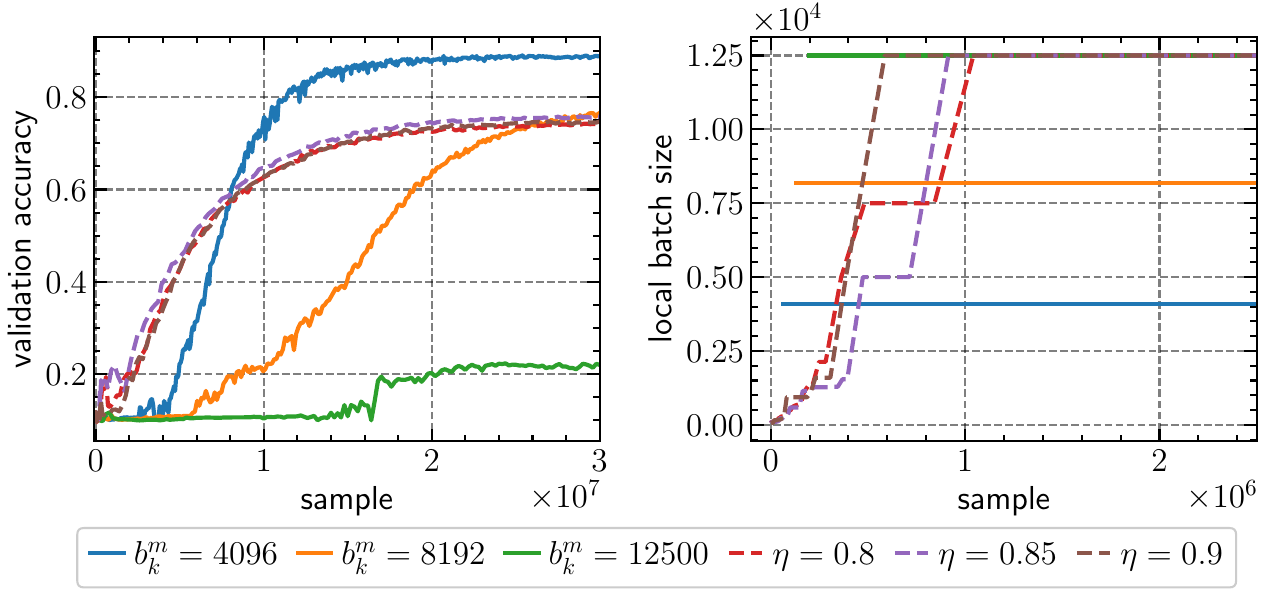}
            \caption{$H=4$ with $b_k^m\in\{4096, 8192, 12500\}$ and $\eta\in\{0.8,0.85,0.9\}$}
            \label{fig:resnet-50_cifar10_const_H4_2}
        \end{subfigure}%
        \hfill
        \begin{subfigure}[h]{0.49\textwidth}
            \centering
            \includegraphics[width=\textwidth]{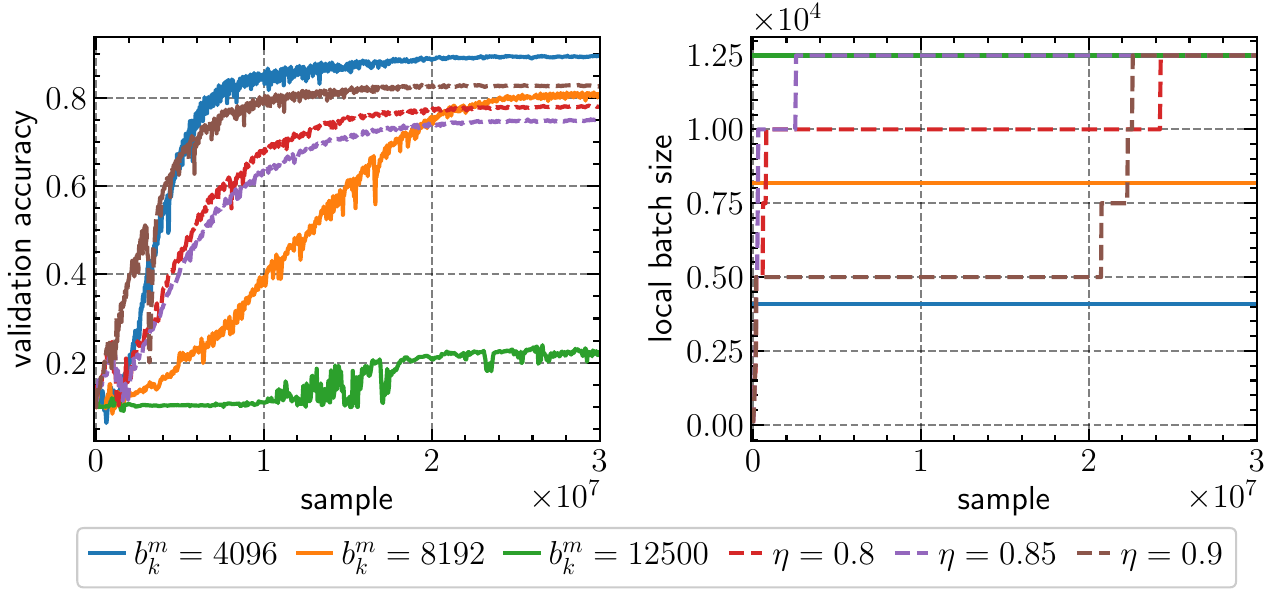}
            \caption{$H=1$ with $b_k^m\in\{4096, 8192, 12500\}$ and $\eta\in\{0.8,0.85,0.9\}$}
            \label{fig:resnet-50_cifar10_const_H1_2}
        \end{subfigure}%
        \hfill
        \begin{subfigure}[h]{0.7\textwidth}
            \centering
            \includegraphics[width=\textwidth]{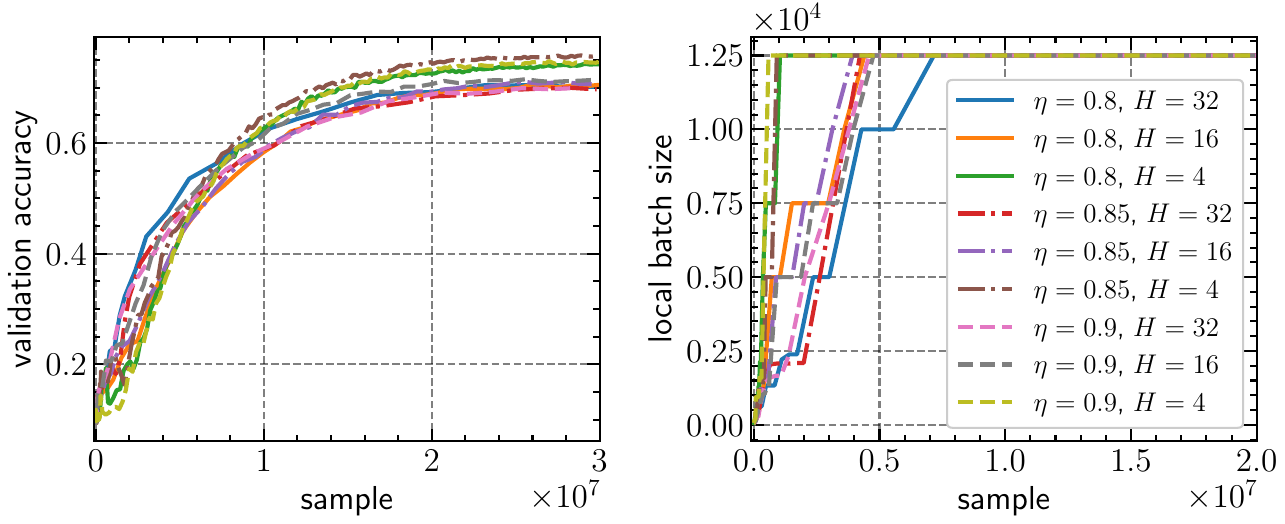}
            \caption{Adaptive batch size strategies for $H\in\{32, 16, 4\}$}
            \label{fig:resnet-50_cifar10_final_2}
        \end{subfigure}
        \caption{Validation accuracy and local batch sizes of Local \SHB with adaptive batch size strategies for \ResNet-50 on CIFAR-10 of the second seed. }
        \label{fig:resnet-50_cifar10_2}
    \end{figure}

    \begin{figure}[h!]    
        \centering
        \begin{subfigure}[h]{0.49\textwidth}
            \centering
            \includegraphics[width=\textwidth]{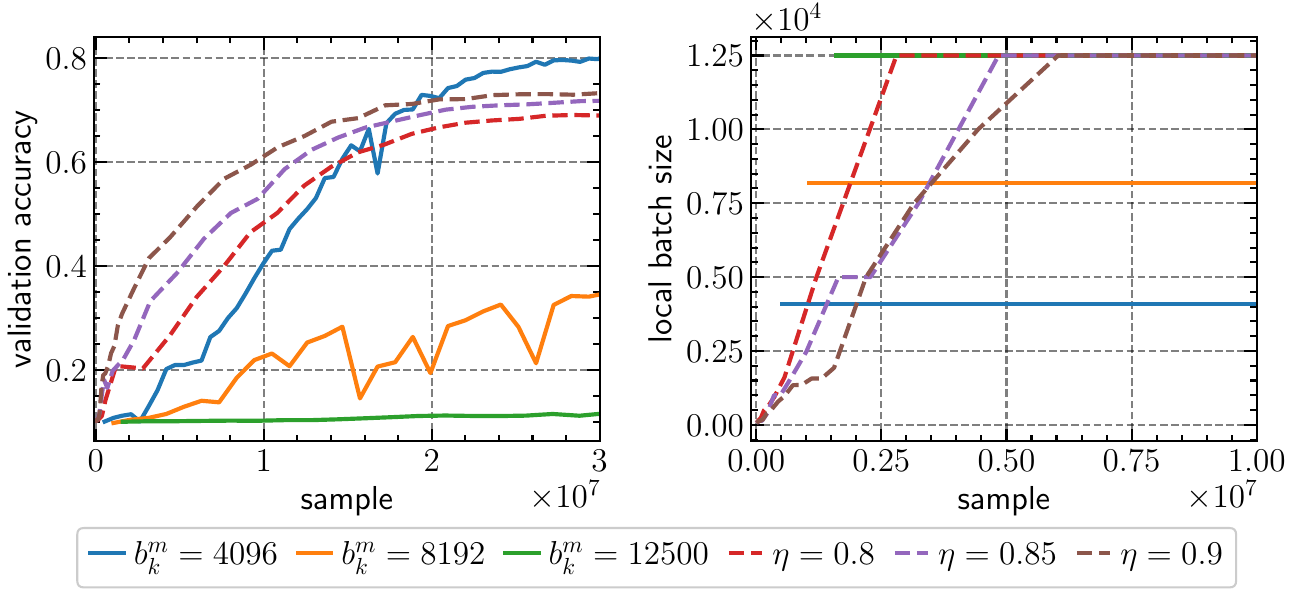}
            \caption{$H=32$ with $b_k^m\in\{4096, 8192, 12500\}$ and $\eta\in\{0.8,0.85,0.9\}$}
            \label{fig:resnet-50_cifar10_const_H32_3}
        \end{subfigure}%
        \hfill
        \begin{subfigure}[h]{0.49\textwidth}
            \centering
            \includegraphics[width=\textwidth]{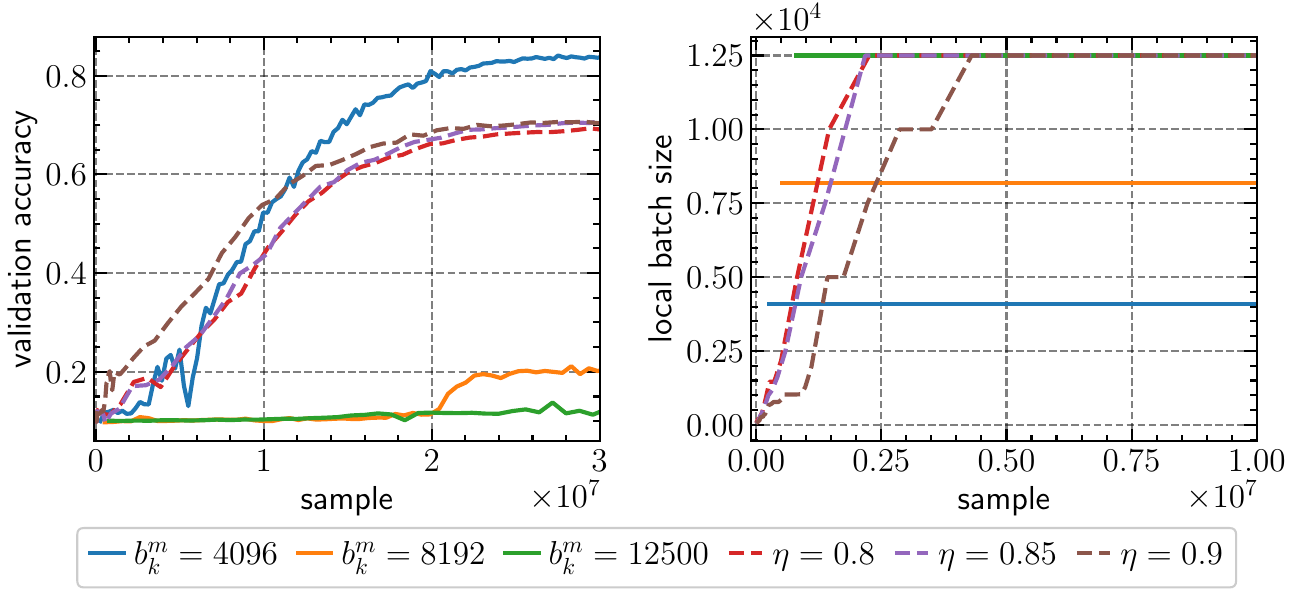}
            \caption{$H=16$ with $b_k^m\in\{4096, 8192, 12500\}$ and $\eta\in\{0.8,0.85,0.9\}$}
            \label{fig:resnet-50_cifar10_const_H16_3}
        \end{subfigure}%
        \hfill
        \begin{subfigure}[h]{0.49\textwidth}
            \centering
            \includegraphics[width=\textwidth]{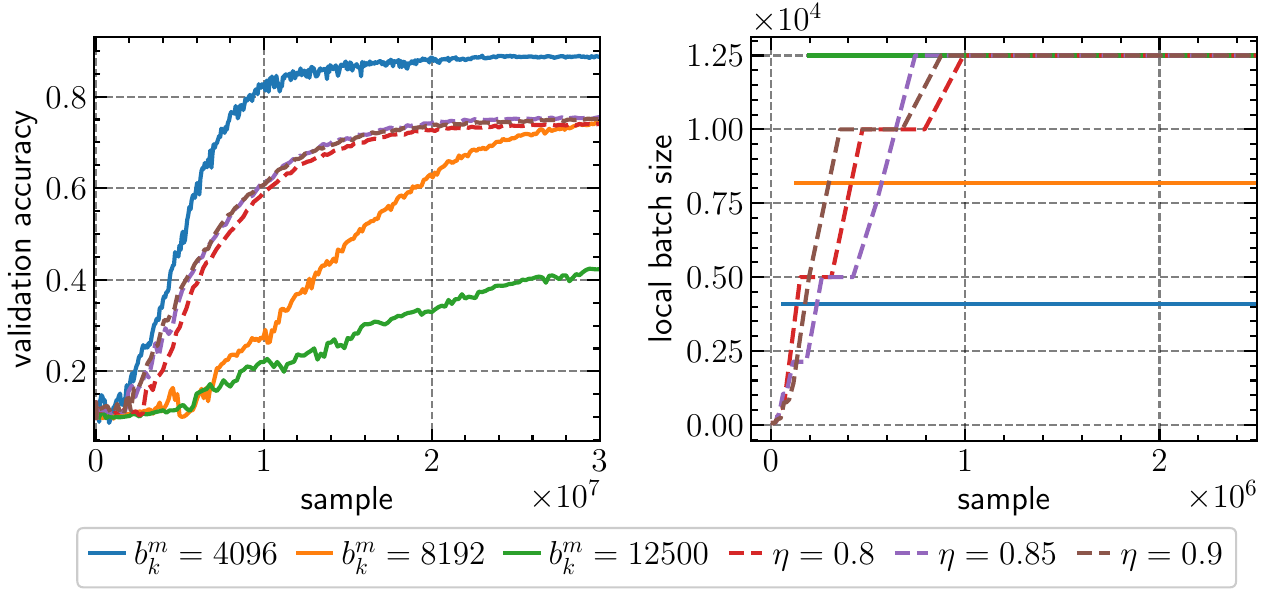}
            \caption{$H=4$ with $b_k^m\in\{4096, 8192, 12500\}$ and $\eta\in\{0.8,0.85,0.9\}$}
            \label{fig:resnet-50_cifar10_const_H4_3}
        \end{subfigure}%
        \hfill
        \begin{subfigure}[h]{0.49\textwidth}
            \centering
            \includegraphics[width=\textwidth]{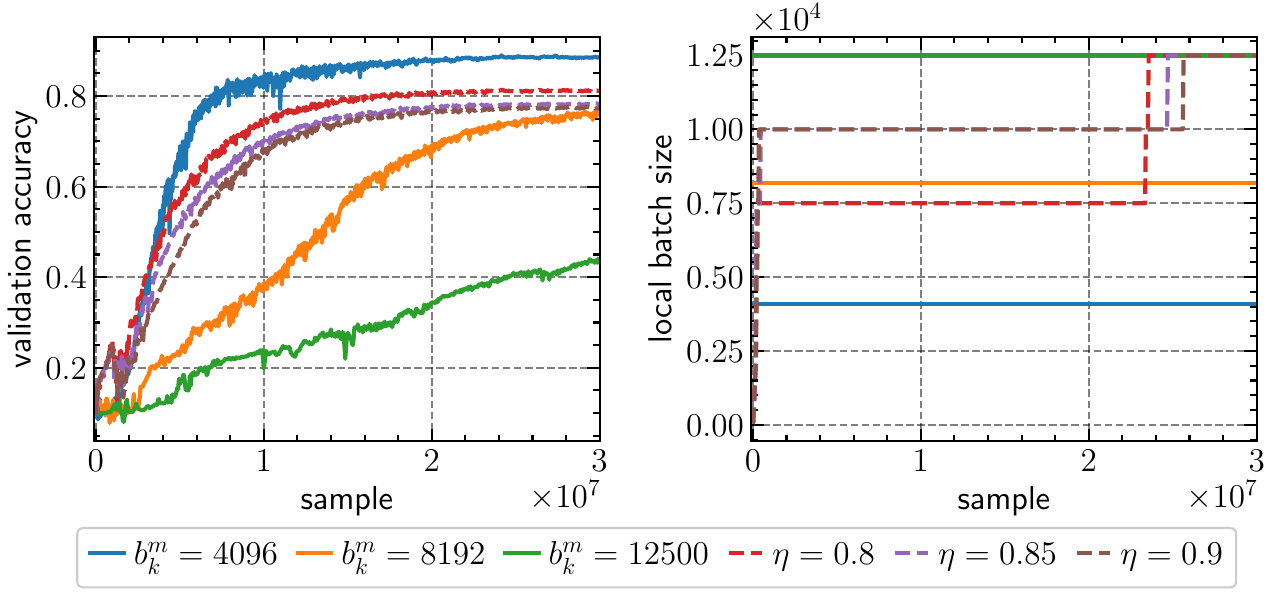}
            \caption{$H=1$ with $b_k^m\in\{4096, 8192, 12500\}$ and $\eta\in\{0.8,0.85,0.9\}$}
            \label{fig:resnet-50_cifar10_const_H1_3}
        \end{subfigure}%
        \hfill
        \begin{subfigure}[h]{0.7\textwidth}
            \centering
            \includegraphics[width=\textwidth]{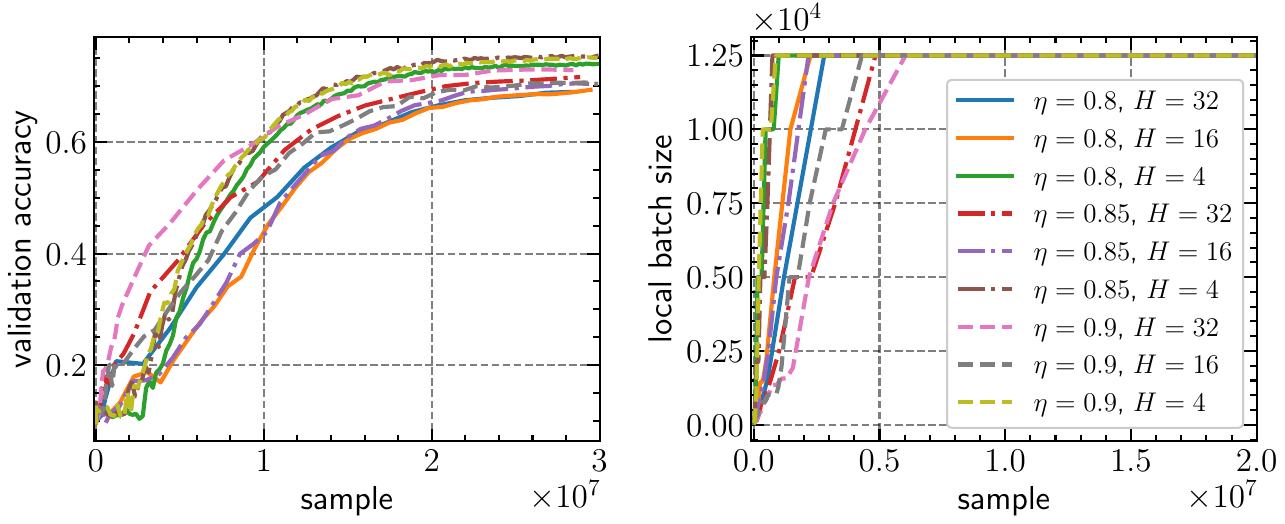}
            \caption{Adaptive batch size strategies for $H\in\{32, 16, 4\}$}
            \label{fig:resnet-50_cifar10_final_3}
        \end{subfigure}
        \caption{Validation accuracy and local batch sizes of Local \SHB with adaptive batch size strategies for \ResNet-50 on CIFAR-10 of the third seed. }
        \label{fig:resnet-50_cifar10_3}
    \end{figure}

    \clearpage
    \newpage
    \subsection{MicroLlama 300M on C4}
    \label{subsec:microllama-300m_c4_supp}
    
    \begin{table}[h!]
        \centering
        \caption{Training hyperparameters for MicroLlama 300M on C4}
        \label{table:hyperparams_microllama-300m_c4}
        \vspace*{2mm}
        \begin{tabular}{lc}
            Model & MicroLlama 300M on C4 \\
            \midrule
            Training samples (sequences) & 2M \\
            Sequence length & 2048 tokens \\
            Weight initialization & Default  \\
            Normalization & Layer Normalization \\
            Optimizer &  \AdamW \\
            $(\beta_1, \beta_2)$ & $(0.9, 0.95)$ \\
            Learning rate schedule & Linear warmup + cosine decay \\
            Learning rate warmup (samples) & 20,000 (1\%) \\
            Peak learning rate & 0.001 \\             
            Base learning rate & 0.0001 \\
            Learning rate scaling rule & linear (for constant batch sizes) \\
            Base global batch size & 256 \\
            Base local batch size & 64 \\
            Maximum global batch size & 8,192 \\
            Maximum local batch size & 2,048 \\
            Weight decay & 0.1 \\
            Weight decay skip bias & Yes \\
            Precision & \texttt{bfloat16} \\
            Dropout & $0$ \\
            Gradient clipping & $1.0$ \\         
            Data-parallel size & 4 \\   
            \bottomrule
        \end{tabular}
    \end{table}

    \begin{table}[h!]
        \centering
        \caption{Adaptive batch size strategies for MicroLlama 300M on C4; mean (standard deviation) of no.~of total steps (steps), wall-clock time (in hours), average local batch sizes (bsz.), and best validation loss (cross-entropy loss; estimated by 100 iterations) of all three seeds}
        \label{table:microllama_full}
        \vspace*{-2mm}
        \scriptsize
        \renewcommand{\arraystretch}{1.}
        \begin{tabular}{l|rrrr|rrrr}
            \toprule[1pt]
            \multirow{2}{*}{Schedule}
            & \multicolumn{4}{c|}{$H=32$}
            & \multicolumn{4}{c}{$H=16$}
            \\
            & steps & time & bsz. & loss & steps & time & bsz. & loss \\
            \midrule
            Constant & 31744 & 10.55 (0.04) & 512 & 4.20 (0.18) & 16384 & 10.51 (0.03) & 512 & 5.88 (0.21) \\
            Constant & 16384 & 10.51 (0.03) & 1024 & 4.96 (0.21) & 7936 & 10.60 (0.60) & 1024 & 5.08 (0.21) \\
            Constant & 8192 & 9.77 (0.01) & 2048 & 5.88 (0.21) & 4096 & 10.47 (0.02) & 2048 & 5.84 (0.10) \\
            $\eta=0.8$ & 14336 (1774) & 11.00 (0.12) & 1162 (155) & 4.81 (0.31) & 6059 (391) & 11.02 (0.05) & 1356 (89) & 4.96 (0.18) \\
            $\eta=0.9$ & 13995 (2132) & 11.15 (0.44) & 1207 (134) & 4.88 (0.19) & 7168 (678) & 11.11 (0.36) & 1163 (130) & 4.72 (0.19) \\
            \bottomrule
        \end{tabular}

        \begin{tabular}{l|rrrr}
            \toprule[1pt]
            \multirow{2}{*}{Schedule}
            & \multicolumn{4}{c}{$H=4$} \\
            & steps & time & bsz. & loss \\
            \midrule
            Constant & 14336 & 11.00 (0.12) & 512 & 4.81 (0.31) \\
            Constant & 1968 & 11.28 (0.04) & 1024 & 5.14 (0.13) \\
            Constant & 992 & 10.92 (0.03) & 2048 & 5.99 (0.04) \\
            $\eta=0.8$ & 1206 (177) & 11.14 (0.02) & 1701 (260) & 5.15 (0.38) \\
            $\eta=0.9$ & 1238 (199) & 11.14 (0.06) & 1668 (303) & 4.94 (0.33) \\
            \bottomrule
        \end{tabular}
        \renewcommand{\arraystretch}{1}
        \vspace*{-2.5mm}
    \end{table}

    \paragraph{Observations.}
    For the task of language modeling with MicroLlama 300M on the C4 dataset, we also perform the training runs with three different random seeds. Similar conclusions to those in the main text can also be made, even though the proposed adaptive batch size strategies lead to different batch size growth patterns when using different random seeds (see \Cref{fig:microllama_223,fig:microllama_323}).

    \begin{figure}[h!]    
            \centering
            \begin{subfigure}[h]{0.49\textwidth}
                \centering
                \includegraphics[width=\textwidth]{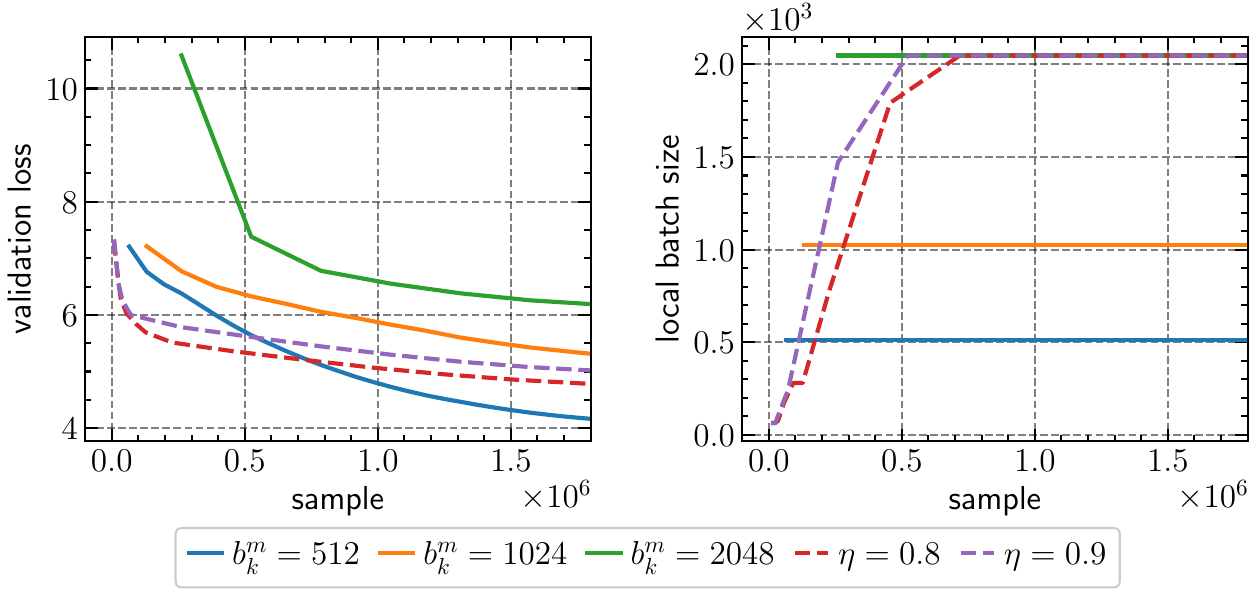}
                \caption{$H=32$ with $b_k^m\in\{512, 1024, 2048\}$ and $\eta\in\{0.8,0.9\}$}
                \label{fig:microllama_const_H32_223}
            \end{subfigure}%
            \hfill
            \begin{subfigure}[h]{0.49\textwidth}
                \centering
                \includegraphics[width=\textwidth]{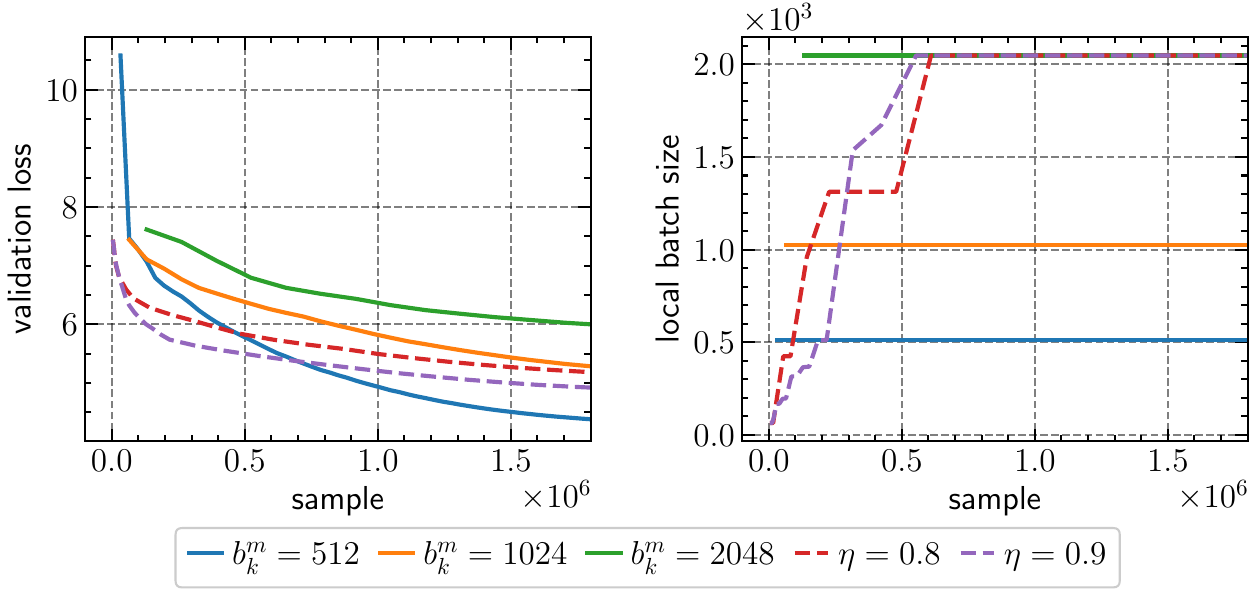}
                \caption{$H=16$ with $b_k^m\in\{512, 1024, 2048\}$ and $\eta\in\{0.8,0.9\}$}
                \label{fig:microllama_const_H16_223}
            \end{subfigure}%
            \hfill
            \begin{subfigure}[h]{0.49\textwidth}
                \centering
                \includegraphics[width=\textwidth]{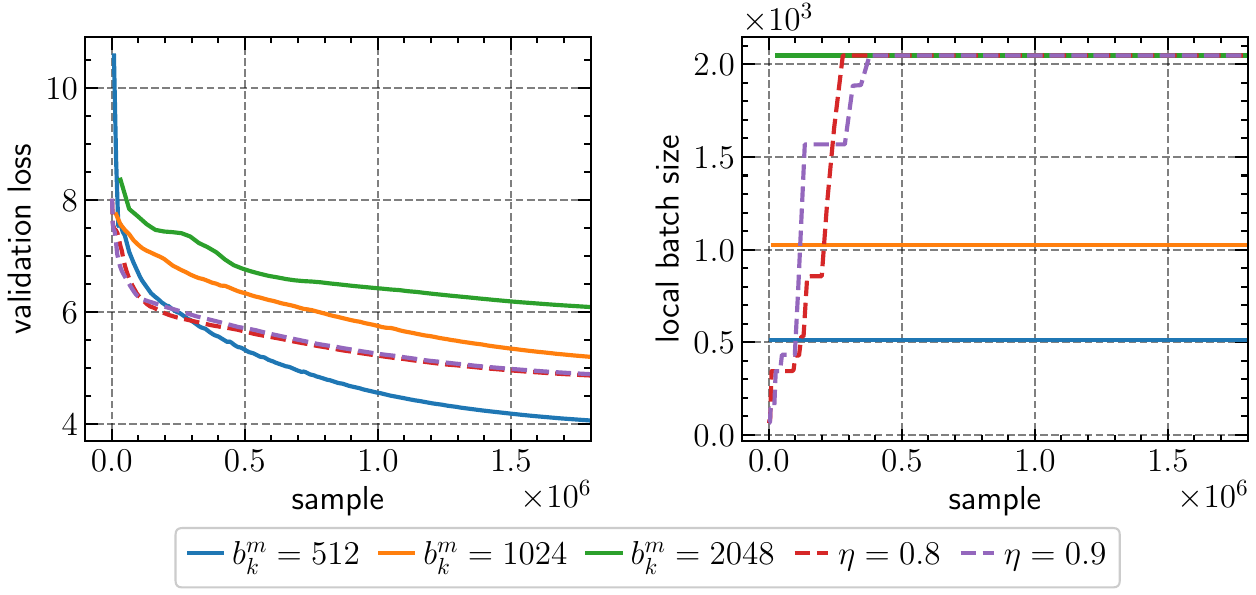}
                \caption{$H=4$ with $b_k^m\in\{512, 1024, 2048\}$ and $\eta\in\{0.8,0.9\}$}
                \label{fig:microllama_const_H4_223}
            \end{subfigure}%
             \hfill
            \begin{subfigure}[h]{0.49\textwidth}
                \centering
                \includegraphics[width=\textwidth]{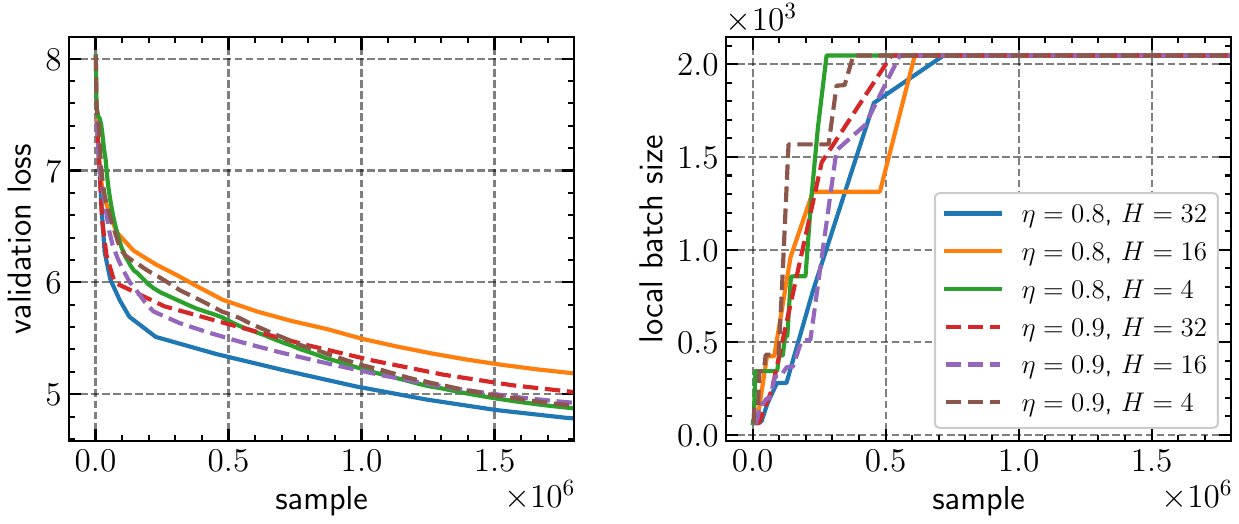}
                \caption{Adaptive batch size strategies for $H\in\{32, 16, 4\}$}
                \label{fig:microllama_final_223}
            \end{subfigure}
            \caption{Validation loss and local batch sizes of Local \AdamW with adaptive batch size strategies for MicroLlama 300M on C4 of the second seed. }
            \label{fig:microllama_223}
        \end{figure}

        \begin{figure}[h!]    
            \centering
            \begin{subfigure}[h]{0.49\textwidth}
                \centering
                \includegraphics[width=\textwidth]{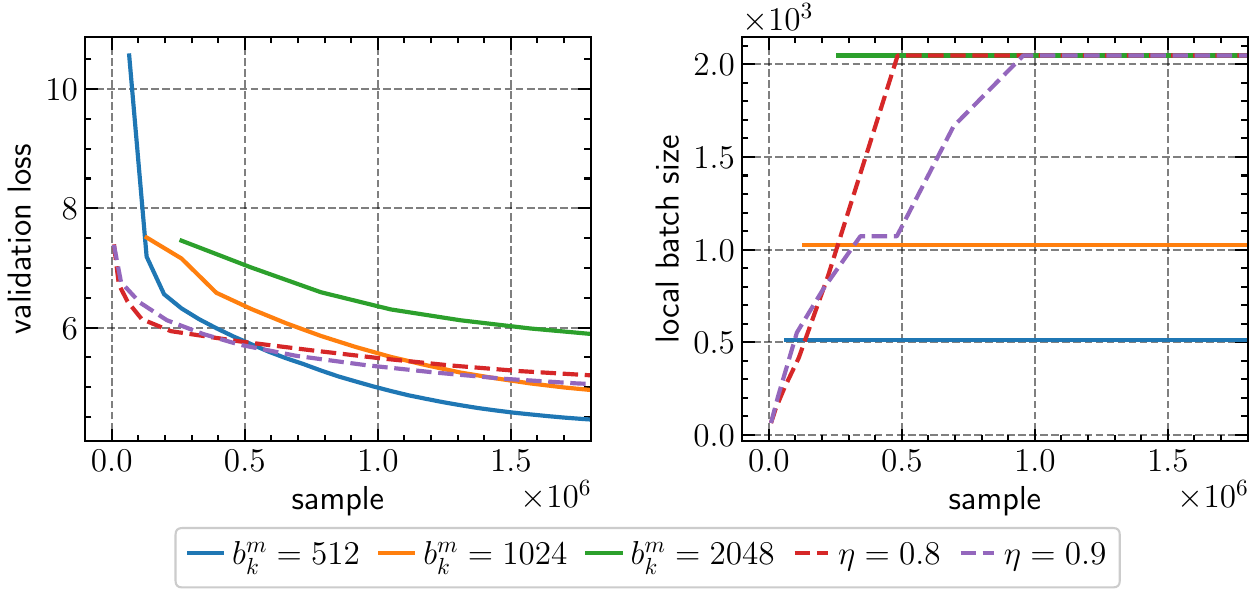}
                \caption{$H=32$ with $b_k^m\in\{512, 1024, 2048\}$ and $\eta\in\{0.8,0.9\}$}
                \label{fig:microllama_const_H32_323}
            \end{subfigure}%
            \hfill
            \begin{subfigure}[h]{0.49\textwidth}
                \centering
                \includegraphics[width=\textwidth]{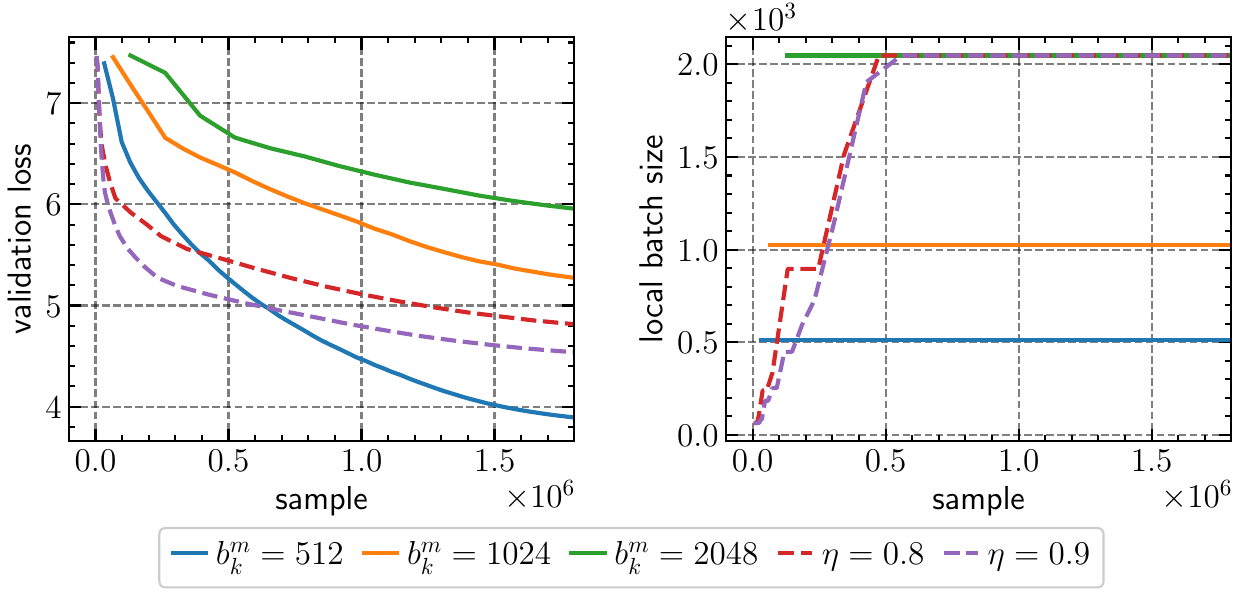}
                \caption{$H=16$ with $b_k^m\in\{512, 1024, 2048\}$ and $\eta\in\{0.8,0.9\}$}
                \label{fig:microllama_const_H16_323}
            \end{subfigure}%
            \hfill
            \begin{subfigure}[h]{0.49\textwidth}
                \centering
                \includegraphics[width=\textwidth]{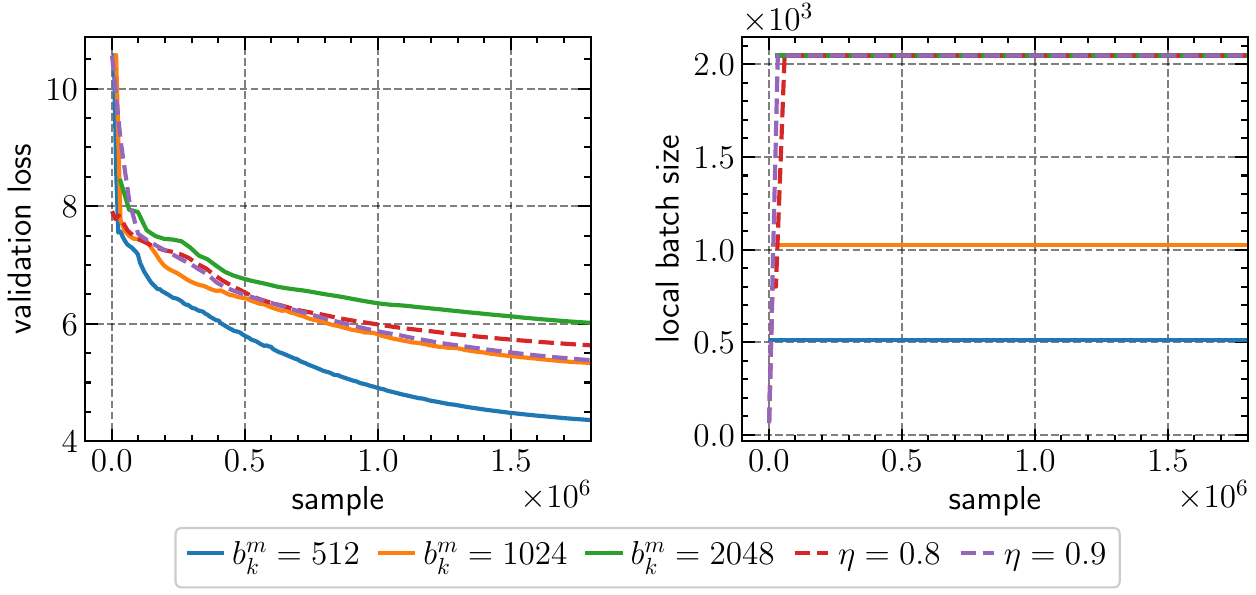}
                \caption{$H=4$ with $b_k^m\in\{512, 1024, 2048\}$ and $\eta\in\{0.8,0.9\}$}
                \label{fig:microllama_const_H4_323}
            \end{subfigure}%
             \hfill
            \begin{subfigure}[h]{0.49\textwidth}
                \centering
                \includegraphics[width=\textwidth]{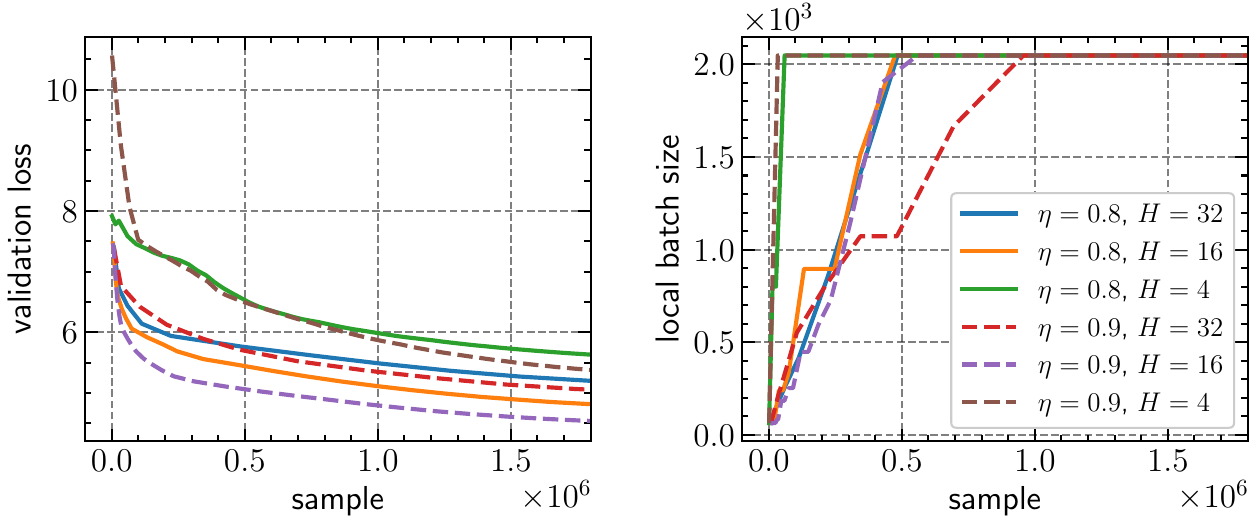}
                \caption{Adaptive batch size strategies for $H\in\{32, 16, 4\}$}
                \label{fig:microllama_final_323}
            \end{subfigure}
            \caption{Validation loss and local batch sizes of Local \AdamW with adaptive batch size strategies for MicroLlama 300M on C4 of the third seed. }
            \label{fig:microllama_323}
        \end{figure}
	
    \clearpage
    \newpage
    \subsection{\ResNet-101 on ImageNet}
    \label{subsec:resnet-101_imagenet_supp}
    
    We also train a \ResNet-101 \citep{he2016deep} on the ImageNet dataset \citep{russakovsky2015imagenet} to evaluate our method on models and datasets of larger scales.

    \begin{table}[h!]
        \centering
        \caption{Training hyperparameters for \ResNet-101 on ImageNet}
        \label{table:hyperparams_resnet-101_imagenet}
        \vspace*{2mm}
        \begin{tabular}{lc}
            Model & \ResNet-101 on ImageNet \\
            \midrule
            Training samples & 256,233,400 (200 epochs) \\
            Weight initialization & Default  \\
            Optimizer &  Momentum \SGD (\SHB) \\
            Learning rate schedule & Linear warmup + cosine decay \\
            Learning rate warmup (samples) & 6,405,835 (2.5\%) \\
            Peak learning rate & 2.5 \\              
            Base learning rate & 0.25 \\
            Base global batch size & 512 \\
            Base local batch size & 128 \\
            Maximum global batch size & 52,000 \\
            Maximum local batch size & 13,000 \\
            Weight decay & 0.0001 \\
            Momentum & 0.9 \\
            Precision & \texttt{bfloat16} \\
            Data-parallel size & 4 \\
            \bottomrule
        \end{tabular}
    \end{table}

    \begin{table}[h!]
        \centering
        \caption{Adaptive batch size strategies for \ResNet-101 on ImageNet; no.~of total steps (steps), wall-clock time (in hours), average local batch sizes (bsz.), and best top-1 validation accuracy (acc.; in \%), and best top-5 validation accuracy (acc.@5; in \%)}
        \label{table:resnet-101_imagenet}
        % \vspace*{-2mm}
        \scriptsize
        % \footnotesize
%        \renewcommand{\arraystretch}{1.}
        \begin{tabular}{l|rrrrr|rrrrr}
            \toprule[1pt]
            \multirow{2}{*}{Schedule}
            & \multicolumn{5}{c|}{$H=32$}
            & \multicolumn{5}{c}{$H=16$}
            \\
            & steps & time & bsz. & acc. & acc.@5 & steps & time & bsz. & acc. & acc.@5 \\
            \midrule
            Constant & 10656 & 14.56 & 6000 & 59.20 & 81.84 & 10672 & 14.78 & 6000 & 63.76 & 85.18 \\
            Constant & 4896 & 14.35 & 13000 & 38.77 & 63.30 & 4912 & 14.34 & 13000 & 50.87 & 74.89 \\
            $\eta=0.9$ & 5216 & 14.53 & 12284 & 50.61 & 74.59 & 5072 & 14.64 & 12603 & 55.63 & 78.86 \\
            $\eta=0.95$ & 5280 & 14.31 & 12124 & 49.13 & 73.23 & 5088 & 15.09 & 12573 & 58.41 & 81.17 \\
            \bottomrule
        \end{tabular}

        \begin{tabular}{l|rrrrr}
            \toprule[1pt]
            \multirow{2}{*}{Schedule}
            & \multicolumn{5}{c}{$H=4$}
            \\
            & steps & time & bsz. & acc. & acc.@5 \\
            \midrule
            Constant & 10676 & 17.20 & 6000 & 71.28 & 89.97 \\
            Constant & 4924 & 15.41 & 13000 & 62.66 & 84.33 \\
            $\eta=0.9$ & 4952 & 15.62 & 12931 & 65.90 & 86.47 \\
            $\eta=0.95$ & 4976 & 16.75 & 12873 & 67.05 & 87.24 \\
            \bottomrule
        \end{tabular}
    \end{table}

    \paragraph{Observations.}
    We observe from \Cref{table:resnet-101_imagenet} that the proposed adaptive batch size strategy with $\eta=0.9$ achieves better generalization performance (in terms of both accuracy and top-5 accuracy) than that of a constant large local batch size of 13,000 but worse than that of a constant small local batch size of 6,000 (i.e., generalization gap), while requiring also an intermediate amount of training time. The performance gap grows with the number of local steps $H$ (cf.~\Cref{fig:resnet_imagenet_H32,fig:resnet_imagenet_H16,fig:resnet_imagenet_H4}). 

    \begin{figure}[h!]    
        \centering
        \includegraphics[width=15cm]{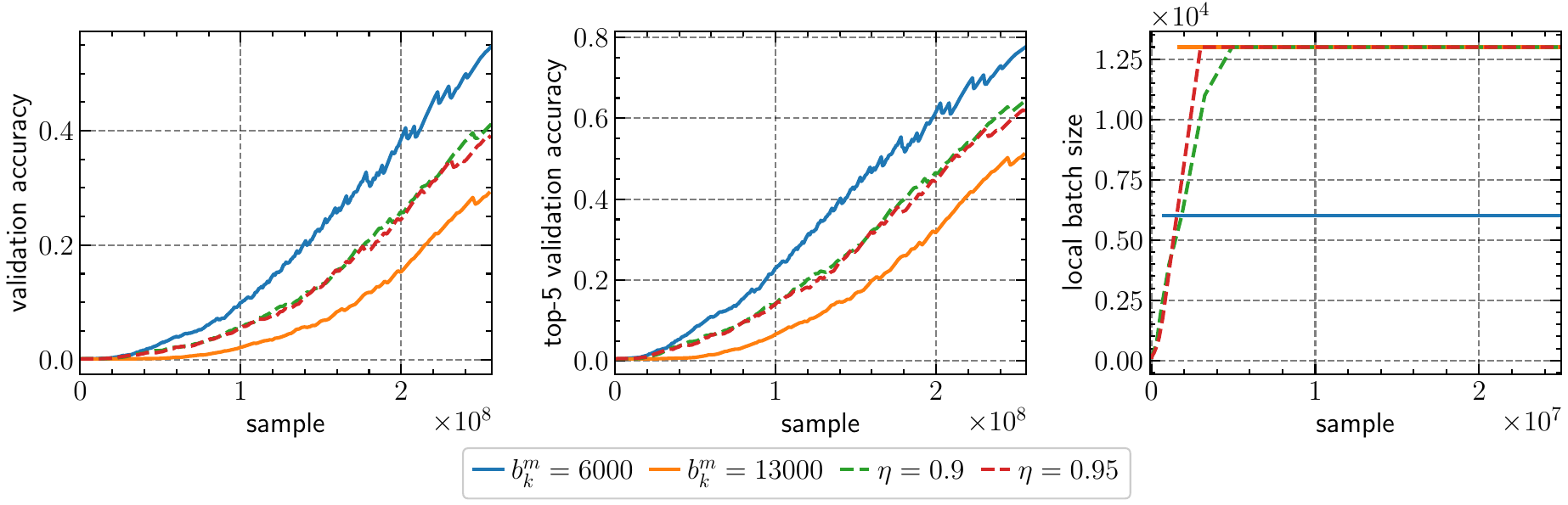}
        \caption{Validation accuracy, top-5 validation accuracy and local batch sizes curves of Local \SHB with adaptive batch size strategies for \ResNet-101 on ImageNet for $H=32$. }
        \label{fig:resnet_imagenet_H32}
    \end{figure}

    \begin{figure}[h!]    
        \centering
        \includegraphics[width=15cm]{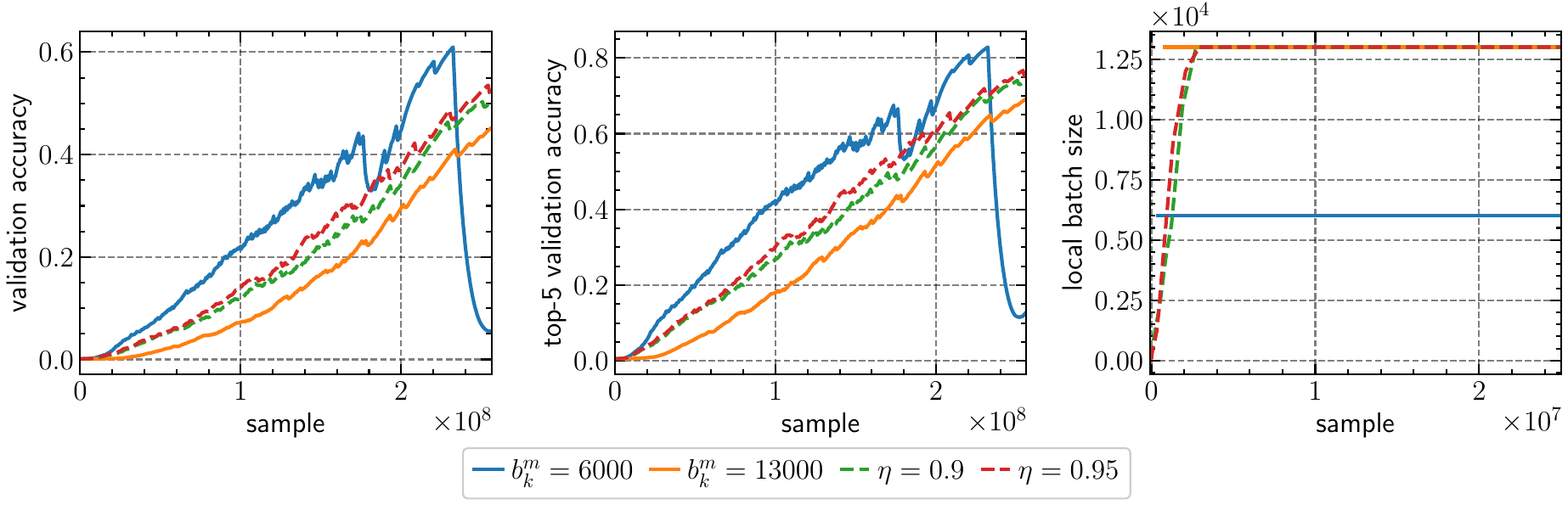}
        \caption{Validation accuracy, top-5 validation accuracy and local batch sizes curves of Local \SHB with adaptive batch size strategies for \ResNet-101 on ImageNet for $H=16$. }
        \label{fig:resnet_imagenet_H16}
    \end{figure}
    
    \begin{figure}[h!]    
        \centering
        \includegraphics[width=15cm]{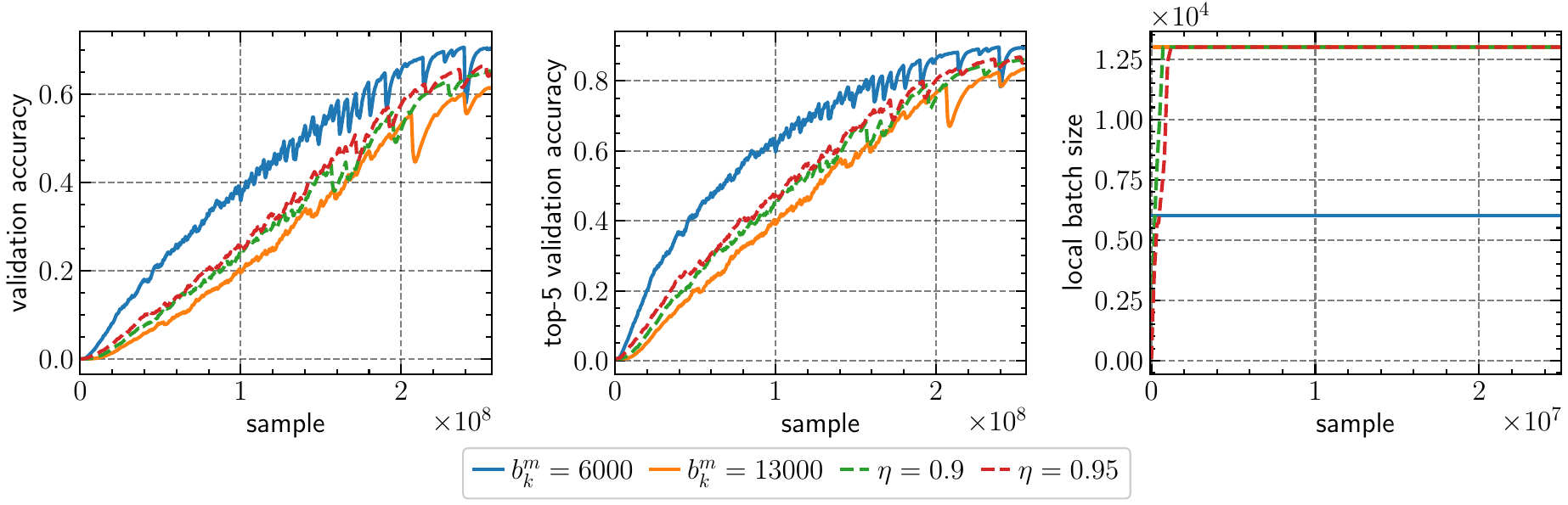}
        \caption{Validation accuracy, top-5 validation accuracy and local batch sizes curves of Local \SHB with adaptive batch size strategies for \ResNet-101 on ImageNet for $H=4$. }
        \label{fig:resnet_imagenet_H4}
    \end{figure}

\end{document}